\documentclass{article}

\PassOptionsToPackage{numbers, compress}{natbib}
\usepackage[preprint]{neurips_2026}

\usepackage[utf8]{inputenc} 
\usepackage[T1]{fontenc}    
\usepackage{hyperref}       
\usepackage{url}            
\usepackage{booktabs}       
\usepackage{amsfonts}       
\usepackage{amsmath}
\usepackage{amssymb}
\usepackage{mathtools}
\usepackage{amsthm}
\usepackage{nicefrac}       
\usepackage{microtype}      
\usepackage{xcolor}         
\usepackage{bm}             
\usepackage{enumitem}       
\usepackage{algorithm}
\usepackage{algorithmic}
\usepackage{cleveref}       
\usepackage[normalem]{ulem}
\usepackage{wrapfig}
\usepackage{subcaption}
\usepackage{tikzsymbols} 

\definecolor{JungleGreen}{RGB}{34,139,34}

\theoremstyle{plain}
\newtheorem{theorem}{Theorem}[section]
\newtheorem{proposition}[theorem]{Proposition}
\newtheorem{lemma}[theorem]{Lemma}

\theoremstyle{definition}
\newtheorem{definition}[theorem]{Definition}

\theoremstyle{remark}

\renewcommand{\S}{\ensuremath{\mathbb{S}}}

\newcommand{\M}{\ensuremath{\mathcal{M}}}
\newcommand{\U}{\ensuremath{\mathcal{U}}}
\renewcommand{\vector}[1]{\mathrm{\mathbf{#1}}}

\title{Spherical Harmonic Optimal Transport: \\Application to Climate Models Comparisons}

\author{%
  Pierre Houédry   \\
  INRIA Rennes \\
  \texttt{pierre.houedry@irisa.fr} \\
  \And 
  Iskander Legheraba \\
  University of Montpellier \\
  LPHI, UMR 5294, CNRS, INSERM\\
  \texttt{iskander-sabri.legheraba@umontpellier.fr} \\
 \AND
  Léo Buecher \\
  Université Bretagne Sud \\
IRISA, UMR 6074, CNRS \\
   \texttt{leo.buecher@irisa.fr} \\
 \And
  Nicolas Courty \\
 Université Bretagne Sud \\
IRISA, UMR 6074, CNRS \\
  \texttt{courty@univ-ubs.fr} \\
}

\begin{document}

\maketitle

\begin{abstract}
Optimal transport provides a powerful framework for comparing measures while respecting the geometry of their support, but comes with an expensive computational cost, hindering its potential application to real world use cases. On manifolds, convolutional algorithms based on the heat kernel have been proposed to alleviate this cost, but their theoretical properties remain largely unexplored. 
We establish that the heat kernel cost converges to the optimal transport cost as time vanishes in the balanced and unbalanced cases.  
In the specific case of the 2-sphere $\mathbb{S}^2$, we ensure that the associated Sinkhorn divergences retains the desirable geometric and analytic properties of classical optimal transport discrepancies. Moreover, we leverage the harmonic structure of the sphere to derive a fast Sinkhorn algorithm, requiring only $\mathcal{O}(n)$ memory and $\mathcal{O}(n^{3/2})$ time per iteration, with fully dense GPU-friendly operations. 
We validate its computational efficiency on synthetic data, and discuss its potential use in the evaluation of global climate models, providing both spatial and seasonal insights into models performances.
\end{abstract}
\section{Introduction}

\paragraph{Numerical OT.}
Optimal transport (OT) provides a principled framework for comparing probability measures that explicitly accounts for the geometry of their support, and has become a central tool in machine learning \cite{villani2009optimal}, with applications spanning generative modeling \cite{arjovskyWassersteinGenerativeAdversarial2017}, domain adaptation \cite{courtyOptimalTransportDomain2017}, and barycenter estimation \cite{cuturi2013sinkhorn, peyre2019computational}.

While data is often assumed to live in Euclidean spaces, many situations exhibit phenomena that inherently lie on curved manifolds. A canonical example is the two-dimensional sphere $\mathbb{S}^2$, which arises naturally in directional statistics \cite{mardia2000directional} and geospatial or climate sciences \cite{garrett2024validating}. This has motivated the development of OT methods that operate directly on the sphere and respect its intrinsic Riemannian geometry \cite{bonetSphericalSlicedWasserstein2022, liuLinearSphericalSliced2024}.

A major algorithmic advance was the introduction of entropic regularization, leading to the Sinkhorn algorithm and enabling scalable OT approximations \cite{cuturi2013sinkhorn}. Entropic regularization also gives rise to the Sinkhorn divergence, which corrects biases inherent to the regularized problem and yields a principled loss function for machine learning under suitable assumptions on the kernel \cite{pmlr-v89-feydy19a}. In Euclidean spaces with quadratic cost, the underlying Gibbs kernel is Gaussian and enjoys well-understood analytical properties. On curved manifolds such as $\mathbb{S}^2$, however, the Gibbs kernel generally fails to be positive definite \cite{jayasumana2015kernel, feragen2015geodesic, da2025invariant}, invalidating these desirable properties. Beyond these theoretical limitations, the most computationally efficient Sinkhorn implementations, such as the recent \textsc{FlashSinkhorn} \cite{flashsinkhorn}, exploit the inner-product structure of the squared Euclidean cost, a decomposition unavailable for the squared geodesic distance $d(x,y)=\arccos(\langle x, y \rangle)^2$ on $\mathbb{S}^2$. On a different level, sliced Wasserstein methods are traditional scalable alternatives. They have been extended to the sphere by projecting onto great circles \cite{bonetSphericalSlicedWasserstein2022} and further explored with alternative slicing techniques \cite{quellmalzSlicedOptimalTransport2023, quellmalz2024parallelly, tranStereographicSphericalSliced2024} or linearised variants \cite{liuLinearSphericalSliced2024}. While efficient at handling large number of samples, they usually do not provide transport plans, are subject to randomness via their Monte-Carlo approximations and their unbalanced versions~\cite{bonet2024slicing} were never considered on the sphere.

An appealing alternative, proposed in the context of convolutional OT \cite{solomon2015convolutional}, is to replace the Gibbs kernel with the heat kernel of the underlying manifold, motivated by Varadhan's formula linking heat diffusion to geodesic distances \cite{varadhan1967}. Yet this approach remains theoretically and computationally underdeveloped on $\mathbb{S}^2$: the conditions guaranteeing well-posedness of the resulting Sinkhorn divergence have not been verified for the heat kernel on the sphere, and existing implementations rely on generic numerical solvers that discard its rich harmonic structure.

\paragraph{OT for Climate models Comparison.} One particular and exciting applications of OT on the two-dimensional sphere is the evaluation of global climate models (GCMs). Evaluating the realism of GCMs is a central
challenge in climate science. The sixth-generation Coupled Model
Intercomparison Project~(CMIP6)~\cite{eyring2016overview} provides
a standardized ensemble of model simulations against which
observational references can be compared. A key difficulty is that
geophysical fields such as precipitations are measures that are
spatially structured, non-negative, and naturally defined on the
globe $\mathbb{S}^2$, yet standard evaluation metrics such as
Root Mean Square Error~(RMSE) or pattern correlation treat them as
flat Euclidean fields, ignoring the spherical geometry and the
distributional nature of precipitation. Recentlty, OT-based metrics have been proposed to address these limitations and provide a more holistic assessment of model performance \cite{Vissio2020evaluating,garrett2024validating,Rodrigues2025climate,Xie2025discovering}. However, these works have relied on balanced OT formulations that do not account for the systematic global mean biases exhibited by GCMs relative to observations, and have not fully leveraged the geometry of the sphere as they mostly work on time series 1D signals~\cite{garrett2024validating,Xie2025discovering}. Our work addresses these gaps by developing an unbalanced OT (UOT) framework based on the heat kernel for evaluating GCMs on $\S^2$, providing both theoretical guarantees and practical tools for this important application.

\paragraph{Contributions.}
Our contributions can be summarized as follows:
\begin{itemize}
    \item We study entropically regularized balanced and unbalanced OT and Sinkhorn divergences. Using the heat kernel as a replacement for the Gibbs kernel associated with the squared geodesic distance, we show that this cost converges to the optimal transport cost in the small-time limit.  In the case of the 2-sphere, we moreover show that it satisfies the key properties required by the Sinkhorn divergence framework.
\item We exploit the spherical harmonic transform to derive \textsc{SHOT}, an efficient Sinkhorn algorithm for measures on $\S^2$ via fast convolutions against the heat kernel, enabling scalable computation of Sinkhorn divergences, with performance validated on synthetic data.
\item We apply our framework to the evaluation of global climate models, comparing the precipitation fields of CMIP6 models against ERA5 reanalysis data. We show that our unbalanced OT approach captures both global and regional biases in model performance, and provides insights into seasonal differences in model skill across different latitudinal bands.
\end{itemize}

\section{Background}\label{sec:background}
Let $\M$ be a compact, connected Riemannian manifold and $C:\M^2 \to \mathbb{R}_+$ be a measurable cost function, which we will typically take to be the squared geodesic distance $d^2$ on $\M$.
\begin{definition}[Optimal Transport] Let $P$ and $Q$ be two probability measures on $\M$.
	The \emph{optimal transport} (OT) problem is
	\begin{equation}\label{eq:ot}
	\mathrm{OT}_C(P,Q) := \inf_{\pi \in \U(P,Q)} \int_{\M^2} C(x,y)\, \mathrm{d}\pi(x,y).
	\end{equation}
  where $\U(P,Q)$ is the set of \emph{couplings} between $P$ and $Q$, i.e., joint probability measures on $\M^2$ with marginals $P$ and $Q$.
\end{definition}

Existence of optimal couplings is well understood~\cite{villani2009optimal,santambrogio2015optimal}.
The choice $C=d^2$ gives $\mathrm{OT}_{d^2}=W_2^2$, the squared $2$-Wasserstein distance.
While the OT problem~\eqref{eq:ot} is convex, it is typically not strictly convex and may
admit multiple optimal couplings, which complicates both analysis and computation.
Moreover, solving~\eqref{eq:ot} exactly between two discrete measures with support of size $n$
requires $\mathcal{O}(n^3)$ operations~\cite{peyre2019computational}, which is prohibitive at scale.
A further limitation is that~\eqref{eq:ot} requires $P$ and $Q$ to have equal total mass,
a constraint often violated in practice when comparing unnormalized densities, point clouds
of different sizes, or measures corrupted by outliers. Both issues are addressed simultaneously by the following formulation.

\begin{definition}[Entropic Unbalanced OT]
For $\varepsilon \geq 0$ and $\tau \in (0, +\infty]$, the \emph{entropic unbalanced OT}
between two positive Radon measures $P$ and $Q$ on $\M$ is
\begin{equation}\label{eq:uot}
  \mathrm{UOT}_{C,\varepsilon}^{\tau}(P,Q)
  := \inf_{\pi \in \mathcal{R}_+(\M^2)}
  \int_{\M^2} C\, \mathrm{d}\pi
  + \varepsilon\, \mathrm{KL}(\pi \mid P \otimes Q)
  + \tau\left( \mathrm{KL}(\pi^1 \mid P)
  + \mathrm{KL}(\pi^2 \mid Q)\right),
\end{equation}
where $\mathcal{R}_+(\M^2)$ is the set of positive Radon measures on $\M^2$, $\mathrm{KL}$ is the Kullback-Leibler divergence (with the convention $+\infty \cdot \mathrm{KL}(\cdot\mid\cdot) = \iota_{\{0\}}(\mathrm{KL}(\cdot\mid\cdot))$) and $\pi^1$ and $\pi^2$ are the first and second marginals of $\pi$, respectively.
\end{definition}

The parameter $\tau$ interpolates between two regimes. At $\tau = +\infty$, exact marginal constraints are enforced, yielding \emph{balanced} entropic OT (and~\eqref{eq:ot} as $\varepsilon \to 0$). For finite $\tau > 0$, the problem becomes \emph{unbalanced}: mass creation or destruction is allowed at a cost proportional to KL divergence from each marginal, with $\tau \to 0$ recovering unconstrained transport.
We write $\mathrm{OT}_{C,\varepsilon} \coloneqq \mathrm{UOT}_{C,\varepsilon}^{+\infty}$ for the balanced case and $\mathrm{UOT}_{C}^{\tau} := \mathrm{UOT}_{C,0}^{\tau}$ for the unregularized unbalanced OT.
A key advantage of~\eqref{eq:uot} is that it admits an efficient Sinkhorn iterative solver with cost $\mathcal{O}(n^2)$ for $n$ support points.

\paragraph{Sinkhorn Divergence.}

Although computationally efficient, entropic regularization introduces a systematic bias: in general, $\operatorname{UOT}_{C,\varepsilon}^{\tau}(Q,Q) \neq 0$. The Sinkhorn divergence~\eqref{eq:sinkhorn-divergence} corrects this bias by subtracting the self-transport terms.

\begin{definition}[\cite{sejourne2019sinkhorn}]
For $\varepsilon > 0$, the \emph{Sinkhorn divergence} (SD) between $P, Q \in \mathcal{R}_+(\M)$ is defined as
\begin{equation}
\label{eq:sinkhorn-divergence}
S_{C,\varepsilon}^{\tau}(P,Q)
\coloneqq \operatorname{UOT}_{C,\varepsilon}^{\tau}(P,Q)
- \tfrac{1}{2}\left(\operatorname{UOT}_{C,\varepsilon}^{\tau}(P,P)
+\operatorname{UOT}_{C,\varepsilon}^{\tau}(Q,Q)\right) + \tfrac{\varepsilon}{2}(m(P)-m(Q))^2,
\end{equation}
where $m(\cdot)$ denotes the total mass of the measure on the manifold $\M$.
\end{definition}

The SD inherits desirable theoretical properties: it defines a symmetric, positive definite, smooth loss function that is convex in each of its input variables and metrizes the convergence in law when some conditions on the space and the cost function are satisfied \cite[Theorem 5 \& 6]{sejourne2019sinkhorn}.

\paragraph{Convolution Optimal Transport} Every iteration of the Sinkhorn algorithm requires the computation of the convolution of a function $f\in L^2(\M)$ against $\mathcal{K}_{C,\varepsilon}(x,y) = \exp(-C(x,y)/\varepsilon)$. For a general kernel $K:\M^2\to\mathbb{R}$, the convolution is defined as the integral operator
\begin{equation}\label{eq:kernel-conv}
  Kf(x) = \int_\mathcal{M} K(x,y)\, f(y)\,\mathrm{d}\sigma(y).
\end{equation}
Although the straightforward way to compute it has a quadratic complexity in the size of the discretization, more efficient algorithms exist in some situations. Indeed, in Euclidean space, the kernel $\mathcal{K}_{d^2,\varepsilon}$ is the standard Gaussian kernel, in such a way that each iteration only requires a simple Gaussian convolution, which can be computed with complexity $\mathcal{O}(n\log n)$ for a grid of size $n$.

On a general Riemannian manifold, the Gibbs kernel $\mathcal{K}_{d^2,\varepsilon}$ has a very similar role, smoothing functions according to the geodesic distance. However a fast implementation of the convolution against it does not exist. \cite{solomon2015convolutional} sidestep this difficulty by approximating
$\mathcal{K}_{d^2,\varepsilon}$ with the heat kernel. 

For $t > 0$, the \emph{heat kernel} $\mathcal{H}_t: \M^2 \to \mathbb{R}$ is defined by the property that, for every $y \in \mathcal{M}$, the function $(t,x) \mapsto \mathcal{H}_t(x,y)$ is the minimal positive solution of the heat equation $\partial_t u(t,x) = \Delta u(t,x)$, with initial condition $u(0,x) = \delta_y(x)$ in the sense of distributions, and where $\Delta$ is the Laplace-Beltrami operator on $\M$.
The heat kernel is smooth, symmetric, and, as $\M$ is complete, strictly positive
for all $t>0$~\cite[Thm.~7.20, Cor.~8.12]{grigoryan2009heat}.

The approximation of the Gibbs kernel with the heat kernel is justified by Varadhan's asymptotic formula~\cite{varadhan1967}, which implies that $\mathcal{H}_{\varepsilon/4} \approx \mathcal{K}_{d^2,\varepsilon} $ for small $\varepsilon$.
This makes it possible to approximate $\mathrm{OT}_{d^2,\varepsilon}$ by
replacing the Gibbs kernel with $\mathcal{H}_{\varepsilon/4}$ in the convolutions involved in Sinkhorn iterations.

With this substitution, each Sinkhorn iteration amounts to diffusing a function
over $\mathcal{M}$ for a short time interval $\varepsilon/4$.
On meshes or two-dimensional manifolds, heat diffusion can be discretized via an
implicit Euler step, reducing each iteration to solving a sparse linear system
involving the discrete Laplace-Beltrami operator.
Pre-factoring this system via a sparse Cholesky decomposition amortizes the cost
across iterations. However, sparse Cholesky is not GPU friendly, as its computation follows an irregular dependency structure which limits parallelism. Also, its sparse memory access patterns lead to poor memory coalescing, preventing efficient use of GPU bandwidth and throughput, hindering its use in large scale OT problems.

\section{Spherical Harmonic Optimal Transport}
\subsection{Theoretical framework}
\paragraph{Heat kernel and UOT.}
We now formalize the connection between heat kernel-regularized OT and classical OT. For $\varepsilon > 0$, define the smooth \emph{time-$\varepsilon$-cost} $\widehat C_\varepsilon  := -\varepsilon \log \mathcal{H}_{\frac{\varepsilon}{4}}$.
The heuristic approximation proposed by \cite{solomon2015convolutional} suggests that $\mathrm{OT}_{d^2,\varepsilon}$ can be approximated by $\mathrm{OT}_{\widehat C_\varepsilon, \varepsilon/4}$. Our main result establishes that this approximation becomes exact as $\varepsilon \to 0$ in the balanced and unbalanced cases. This is not a direct consequence of standard entropic OT convergence results (e.g. those of~\cite{carlier2017convergence, nutz2021introduction, carlier2023convergence}), as the cost function itself varies with $\varepsilon$.

\begin{theorem}\label{thm:conv-heat-reg}
	For all $P, Q \in \mathcal{R}_{+}(\M)$ and $\tau \in (0, +\infty]$ we have
	\[ \lim_{\varepsilon \to 0} \mathrm{UOT}_{\widehat C_\varepsilon, \varepsilon/4}^{\tau}(P,Q) = \mathrm{UOT}_{d^2}^{\tau}(P,Q). \]
\end{theorem}
The proof is provided in the appendix \ref{appendix:proof_convergence}. This rigorously justifies heat kernel-based Sinkhorn as an efficient proxy for classical OT.
We note $\hat{S}^\tau_\varepsilon := S^\tau_{\widehat C_{\varepsilon},\varepsilon/4}$ for brevity ($\hat{S}_\varepsilon$ in the balanced case).


The following result justifies the use of the Sinkhorn divergence theory in our setting \textit{i.e.} $\M=\S^{2}$, ensuring $\hat{S}^\tau_\varepsilon$ retains its desirable properties.
The proof is deferred to Appendix~\ref{app:proof:sinkh-cond}. 

\begin{proposition}\label{prop:sinkh:cond}
For any $\varepsilon > 0$ and $\tau \in (0, \infty]$, $\hat{S}^\tau_\varepsilon$ defines a symmetric, positive definite, smooth function that is convex in each of its input variables and metrizes the convergence in law.
\end{proposition}

Taken together, Theorem~\ref{thm:conv-heat-reg} and Proposition~\ref{prop:sinkh:cond}
provide a solid theoretical foundation for heat kernel-regularized OT on $\S^2$.
We therefore propose to leverage spherical harmonics, which diagonalize the
Laplace-Beltrami operator on $\S^2$ and reduce heat diffusion to pointwise
multiplication in frequency space, enabling fully parallelizable, GPU-friendly
Sinkhorn iterations.

\paragraph{Spherical Harmonics.}

Functions in \(L^2(\S^2)\) admit an orthonormal basis given by the \emph{spherical harmonics (SH)} \(\{Y_{\ell m}\}_{\ell \ge 0,\, |m| \le \ell}\). Any \(f \in L^2(\S^2)\) admits the expansion
\[
f = \sum_{\ell=0}^{\infty} \sum_{m=-\ell}^{\ell} \mathcal{F}[f](\ell, m)\, Y_{\ell m},
\qquad
\mathcal{F}[f](\ell, m) = \langle f, Y_{\ell m} \rangle_{L^2(\S^2)}.
\]
The transformation \(f \mapsto \mathcal{F}[f]\) is the \emph{spherical harmonic transform (SHT)}. We refer to Appendix \ref{app:spherical-harmonics} for a brief review.

On the sphere, convolution is defined via rotations rather than translations. Following \cite{driscoll1994computing}, given functions $f, \kappa \in L^2(\S^2)$, their spherical convolution is defined as: 

\begin{equation}
    (f \star \kappa)(x)
    = \int_{\mathrm{SO}(3)} f(R\vec{n}) \, \kappa(R^{-1}x) \, \mathrm{d}R,
    \quad \forall\, x \in \S^2,
    \label{eq:spherical-conv}
\end{equation}
where $\vec{n} = (0,0,1)^\top$ denotes the north pole and $\mathrm{d}R$ the (unnormalized) volume measure on $\mathrm{SO}(3)$. 

The spherical convolution admits a convolution theorem.

\begin{theorem}[Spherical Convolution Theorem~\cite{driscoll1994computing}]\label{thm:spherical-conv-theorem}
Let $f, \kappa \in L^2(\S^2)$.
Then
\begin{equation}
\forall l \geq 0, \forall |m| \le \ell, ~	\mathcal{F}[f \star \kappa](\ell,m)
	= 2\pi \sqrt{\frac{4\pi}{2\ell+1}}
	\, \mathcal{F}[f](\ell,m)
	\, \mathcal{F}[\kappa](\ell,0).
	\label{eq:spherical-conv-theorem}
\end{equation}
\end{theorem}

A key observation is that the heat kernel is \emph{radial}, i.e.\ depends only on $\langle x,y\rangle$. For such kernels, the generic kernel operator \eqref{eq:kernel-conv} reduces to a spherical convolution \eqref{eq:spherical-conv}. This contribution differs from prior work  \cite{kondor2018generalization, sosnovik2021disco} in that it provides a direct link between the two convolution types, enabling the use of SHTs for efficient implementation of Sinkhorn iterations with the heat kernel. The proof is provided in Appendix~\ref{app:proof:kern-to-fourier}.

\begin{proposition}\label{prop:kern-to-fourier}
	Let $K$ be a radial kernel on $\S^2 \times \S^2$ of the form $K(x,y) = K_0(\langle x, y \rangle)$ with $K_0 \in L^2([-1,1])$. Then for any $f \in L^2(\S^2)$, writing $\tilde K(z) = K(z, \vec{n})$ with $\vec{n}$ the north pole:
	\[
	\forall x \in \S^2, \quad (Kf)(x) = \frac{1}{2\pi} \big( f \star \tilde K \big)(x).
	\]
\end{proposition}\label{sec:theory}
\subsection{Algorithm}
Each SH \,$Y_{\ell m}$ is an eigenfunction of the Laplace-Beltrami operator \textit{i.e.}
$-\Delta Y_{\ell m} = \ell(\ell+1)\, Y_{\ell m}$,
so the heat operator acts diagonally: $\mathcal{F}[\mathcal{H}_t \star f](\ell,m) = e^{-t\ell(\ell+1)}\,\mathcal{F}[f](\ell,m)$. By Proposition~\ref{prop:kern-to-fourier} and the Spherical Convolution Theorem~\ref{thm:spherical-conv-theorem}, applying the heat kernel to a function reduces to a SHT, pointwise multiplication by the precomputed weights $\{e^{-t\ell(\ell+1)}\}_\ell$, and an inverse SHT. We thus define the \emph{heat convolution} as
\[
  \mathrm{HeatConv}(\vector{v};\,\{h_\ell\}) := \mathcal{F}^{-1}\bigl(\{h_\ell\} \odot \mathcal{F}(\vector{v})\bigr),
  \qquad h_\ell = e^{-t\ell(\ell+1)},
\]
where $\odot$ denotes pointwise multiplication over SH coefficients.  For numerical stability, classical stabilization techniques such as centering dual variables~\cite{sejourne2023unbalanced} can be applied.
This yields Algorithm~\ref{alg:sinkhorn}. Note that in order to match the standard Sinkhorn algorithm with regularization parameter $\varepsilon$, one must set the heat time to $t = \varepsilon/4$, for the theoretical reasons discussed in Section~\ref{sec:theory}.

\paragraph{Discrete Setting and Complexity.}

In order to perform numerical computations, we discretize the sphere and represent
probability distributions as vectors. We consider two discretizations: the HEALPix
grid~\cite{gorski2005healpix}, which provides a hierarchical, equal-area partitioning
of the sphere, and the equiangular Clenshaw-Curtis grid~\cite{clenshaw1960} (see Appendix~\ref{appendix:discretization_sphere} for more details on the discretizations).

Given a discretization $\{x_i\}_{i=1}^n$ with area weights $\vector{a} \in \mathbb{R}^n_+$,
we approximate integrals as $\int_{\S^2} f\,\mathrm{d}\sigma \approx \vector{a}^\top \vector{f}$ (see Appendix~\ref{appendix:discretization_ot}).
The usual Sinkhorn algorithm requires a memory cost of $\mathcal{O}(n^2)$ and a computational cost of $\mathcal{O}(n^2)$ per iteration.
By contrast, our approach requires only $\mathcal{O}(n)$ memory
for both grids, since no kernel matrix is ever formed.
On the Clenshaw-Curtis grid with $n = \mathcal{O}(L^2)$ (where $L$ is the spherical harmonic bandwidth), each SHT costs $\mathcal{O}(L^3) = \mathcal{O}(n^{3/2})$
operations; HEALPix with $n = 12N_\mathrm{side}^2$ (where $N_\mathrm{side}$ is the resolution parameter) achieves the same asymptotic
cost via pseudo-spectral methods.
In both cases, all operations, SHTs and pointwise multiplications, are dense and
parallelizable, making the algorithm fully GPU-friendly. We provide further details on the practical implementation in Appendix~\ref{appendix:implementation_details}.

\section{Experiments}
\subsection{Runtime comparison}

\paragraph{Experimental setup.}
We benchmark all methods in the balanced case (\textit{i.e.} $\tau = \infty$), as some baselines are restricted to this setting. We consider pairs of random probability measures on $\S^2$, generated by normalizing a uniform random field with respect to the area weights of the grid.
Experiments are run on a single NVIDIA A40 GPU with \texttt{float64} precision. Each runtime is averaged over 5 independent runs.
We consider both grids using powers of 2: $L \in \{8, \ldots, 512\}$ for the equiangular (Clenshaw-Curtis) grid (up to ${\approx}523{,}000$ points) and $N_\mathrm{side} \in \{8, \ldots, 256\}$ for HEALPix.
We compare six methods.
The first three use entropic regularization with $\varepsilon = 0.1$ and convergence threshold $10^{-9}$: \textsc{SHOT} (ours), \textsc{Sinkhorn} (with matrix kernel convolutions), and \textsc{Cholesky} \cite{solomon2015convolutional} (convolutional approach with a pre-factored Cholesky decomposition of the Laplace-Beltrami operator, reimplemented in \texttt{PyTorch} \cite{paszke2019pytorch}).
Two sliced methods using 500 projections: \textsc{SSW} \cite{bonetSphericalSlicedWasserstein2022} and \textsc{LSSOT} \cite{liuLinearSphericalSliced2024}.
Finally, $W_2^2$ solves the exact OT linear program on pairwise geodesic distances via \texttt{POT} \cite{POT}.
Note that for \textsc{SHOT} the restriction to the balanced case does not affect much the behavior of the method, as the only difference between the balanced and unbalanced formulations is an additional pointwise power step applied at each Sinkhorn iteration.

\paragraph{Results.}

\begin{figure}[t]
    \centering
    \includegraphics[width=0.8\textwidth]{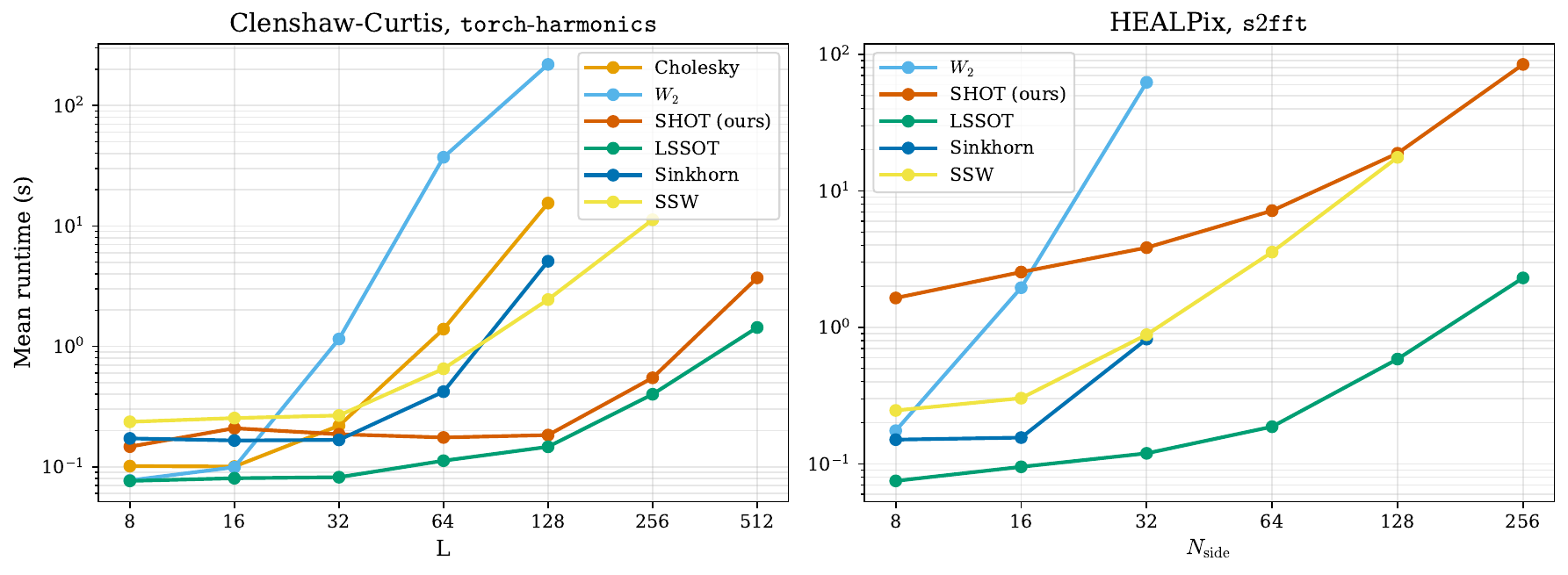}
    \caption{Mean runtime (seconds) vs.\ grid resolution averaged over 5 runs; missing values indicate out-of-memory.
\emph{Left}: Clenshaw-Curtis grid.
\emph{Right}: HEALPix grid.}
    \label{fig:runtime}
\end{figure}

A key advantage of \textsc{SHOT} is its ability to scale to very high resolutions on both grids without running out of memory.
As visible in Figure~\ref{fig:runtime}, methods relying on explicit kernel matrices (\textsc{Sinkhorn}, \textsc{Cholesky}, $W_2^2$) hit out-of-memory errors at moderate resolutions, whereas \textsc{SHOT} successfully runs at the largest resolutions considered. On the Clenshaw-Curtis grid, \textsc{SHOT} is the second fastest method for $L \geq 64$, slightly behind \textsc{LSSOT}.
On the HEALPix grid, however, \textsc{SHOT} does not achieve the same relative speedup despite its favorable theoretical complexity.
We attribute this gap to implementation differences: spherical convolutions on the Clenshaw-Curtis grid are performed via \texttt{torch-harmonics}~\cite{bonev2023spherical}, whereas on HEALPix they rely on \texttt{s2fft}~\cite{price2024differentiable}, which do not provide optimized CUDA kernels. We provide more runtime comparisons in Appendix~\ref{appendix:runtime}, showing that this behavior is consistent across regularization parameters $\varepsilon$ and grid types.

\subsection{Evaluating Climate Models on the Sphere}

\subsubsection{Setup and Motivation}
\label{sec:setup}

We apply our \textsc{SHOT} method to compute the associated SD $\hat{S}^\tau_\varepsilon$ introduced in Section \ref{sec:theory} to evaluate eight CMIP6 models against the
ERA5 reanalysis~\cite{hersbach2020era5}, which we treat as the
observational reference. In this study, we focus on precipitation fields which are represented as
discrete measures on the \textsc{HEALPix} grid~\cite{gorski2005healpix}
at resolution $N_\mathrm{side} = 64$ ($\sim$49\,152 equal-area pixels,
$\approx 1$\textdegree effective resolution), which provides an
equal-area pixelation of $\mathbb{S}^2$ directly compatible with our
algorithm. ERA5 data is accessed via the HERA5
dataset\footnote{\url{https://orcestra-campaign.org/hera5.html}}, already provided in \textsc{HEALPix} format,
while CMIP6 model outputs are retrieved from the Pangeo Google Cloud
archive\footnote{\url{https://pangeo-data.github.io/pangeo-cmip6-cloud/}} and regridded to the same
\textsc{HEALPix} grid via bilinear interpolation. All fields are
converted to mm/day. We analyze two
climatological seasons: December-January-February~(DJF) and
June-July-August~(JJA), each computed as a temporal average over
the 2000--2014 historical period. Visual illustrations of the resulting
precipitation distributions for ERA5 and one model (ACCESS-CM2) are shown in
Appendix~\ref{appendix:iclimate}.

\paragraph{Usefulness of UOT.}
CMIP6 models exhibit systematic global mean precipitation biases
relative to ERA5, so the total mass of the two measures differs. A
balanced transport plan is then forced to compensate for these
intensity differences through spatial displacement, confounding the
two sources of error. The SD $\hat{S}^\tau_\varepsilon$ yields a
principled decomposition of model bias into a \emph{spatial}
component (transport of precipitation from wrong locations) and an
\emph{intensity} component (local over- or under-estimation of
amounts), which is not achievable with any pointwise metric.
The key advantage of \textsc{SHOT} is that it can efficiently compute $\hat{S}^\tau_\varepsilon$ at high resolution while other methods like \textsc{Cholesky} and \textsc{Sinkhorn} require extensive computational ressources to compute their associated SD at the considered resolution. We set $\varepsilon = 0.01$ (squared geodesic distance on the unit sphere) throughout; results are qualitatively stable over $\varepsilon \in [0.01, 1]$.

\subsubsection{Global Model Ranking and Regional Bias Attribution}
\label{sec:ranking}

\paragraph{Scalar ranking.}
For each model $m$ and season, we compute the balanced $\hat{S}_\varepsilon$ and unbalanced  $\hat{S}^\tau_\varepsilon$ (with $\tau=100$) Sinkhorn
divergences between the precipitation measures $\mu_m$ and the reference $\mu_\mathrm{ERA5}$. In the balanced case, both measures are normalized to share the same total mass, while in the unbalanced case they are left unchanged. The resulting scalar score provides a
geometrically meaningful ranking of models that accounts for both the
spatial structure and, in the unbalanced case, the intensity of precipitation
biases. Unlike RMSE, this score is a proper divergence as it equals zero if and
only if $\mu_m = \mu_\mathrm{ERA5}$ a.e. and it metrizes weak
convergence on $\mathbb{S}^2$, meaning that a model improving its
score is genuinely moving its precipitation distribution closer to
the reference in a transport sense.

\begin{wrapfigure}{r}{0.45\linewidth}
  \centering\includegraphics[width=\linewidth]{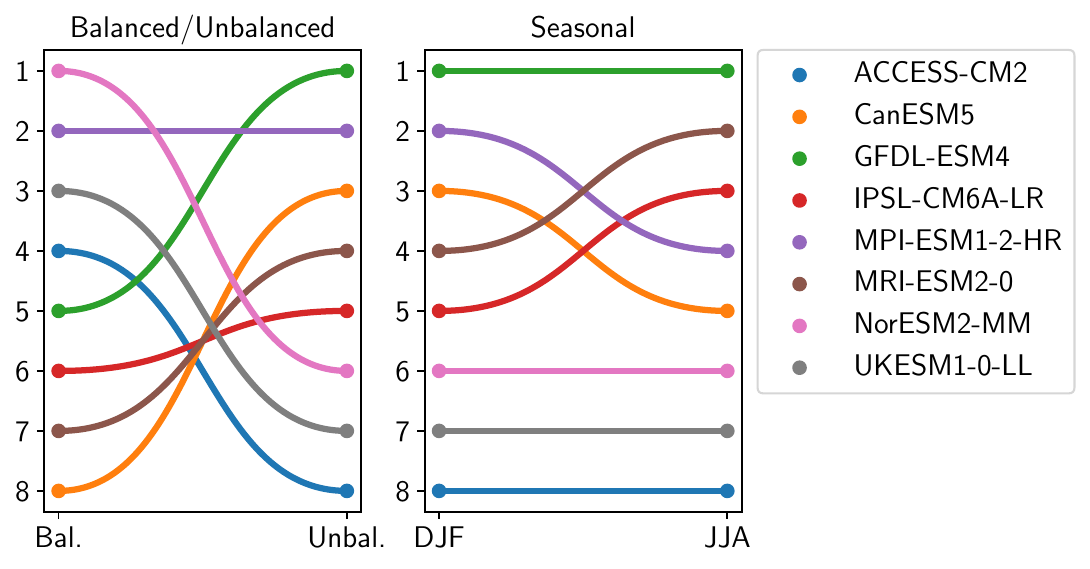}
  \caption{
Model rankings from SD. \textit{Left:} Balanced vs.\ unbalanced Sinkhorn divergences (DJF). \textit{Right:} DJF vs.\ JJA unbalanced rankings.
  }
  \label{fig:rankings}
\end{wrapfigure}

Interestingly, rankings exhibit severe changes between the balanced and unbalanced cases.
Figure~\ref{fig:rankings} (panel 1) shows
differences in those rankings for the two considered configurations
and for all eight considered models in the DJF season. Large rank shifts reveal
models whose scores under balanced OT are dominated by global
intensity biases rather than spatial fidelity: GFDL-ESM4 improves
from rank~5 to rank~1 (intensity bias masking good spatial
structure), while NorESM2-MM degrades from rank~1 to rank~6
(balanced OT artificially rewarding compensating errors).
MPI-ESM1-2-HR is the most metric-stable top performer (rank~2
in both). Figure~\ref{fig:rankings} (panel 2) exhibits seasonal
comparison of unbalanced rankings (DJF vs.\ JJA). Four models have
perfectly stable rankings across seasons, identifying them as the models with the most
season-independent spatial fidelity and bias. Models improving
in JJA (MRI-ESM2-0, IPSL-CM6A-LR) suggest boreal-summer specific improvements in spatial precipitation
structure. We then turn to the question of \emph{where} on the globe
these improvements occur, which is not revealed by the scalar rankings but can be addressed
by the gradient-based regional attribution described next.

\paragraph{Gradient-based regional attribution.}
The SD scalar score summarizes model quality but
does not reveal \emph{where} on the globe the bias originates. To
address this, we exploit the fact that $\hat{S}^\tau_\varepsilon$
is differentiable with respect to $\mu_m$. Its gradient
$g_m = \nabla_{\mu_m} \hat{S}^\tau_\varepsilon(\mu_m,
\mu_\mathrm{ERA5})$ is a signed field on $\mathbb{S}^2$, derived from
the Kantorovich potentials of the dual problem, where $g_m(x) > 0$
indicates local excess precipitation mass and $g_m(x) < 0$ indicates
deficit relative to the optimal transport plan. Since Kantorovich
potentials are defined up to an additive constant, we center $g_m$ by
subtracting its area-weighted mean, retaining only the spatially
varying component.

\begin{figure}
  \centering\includegraphics[width=0.6\linewidth]{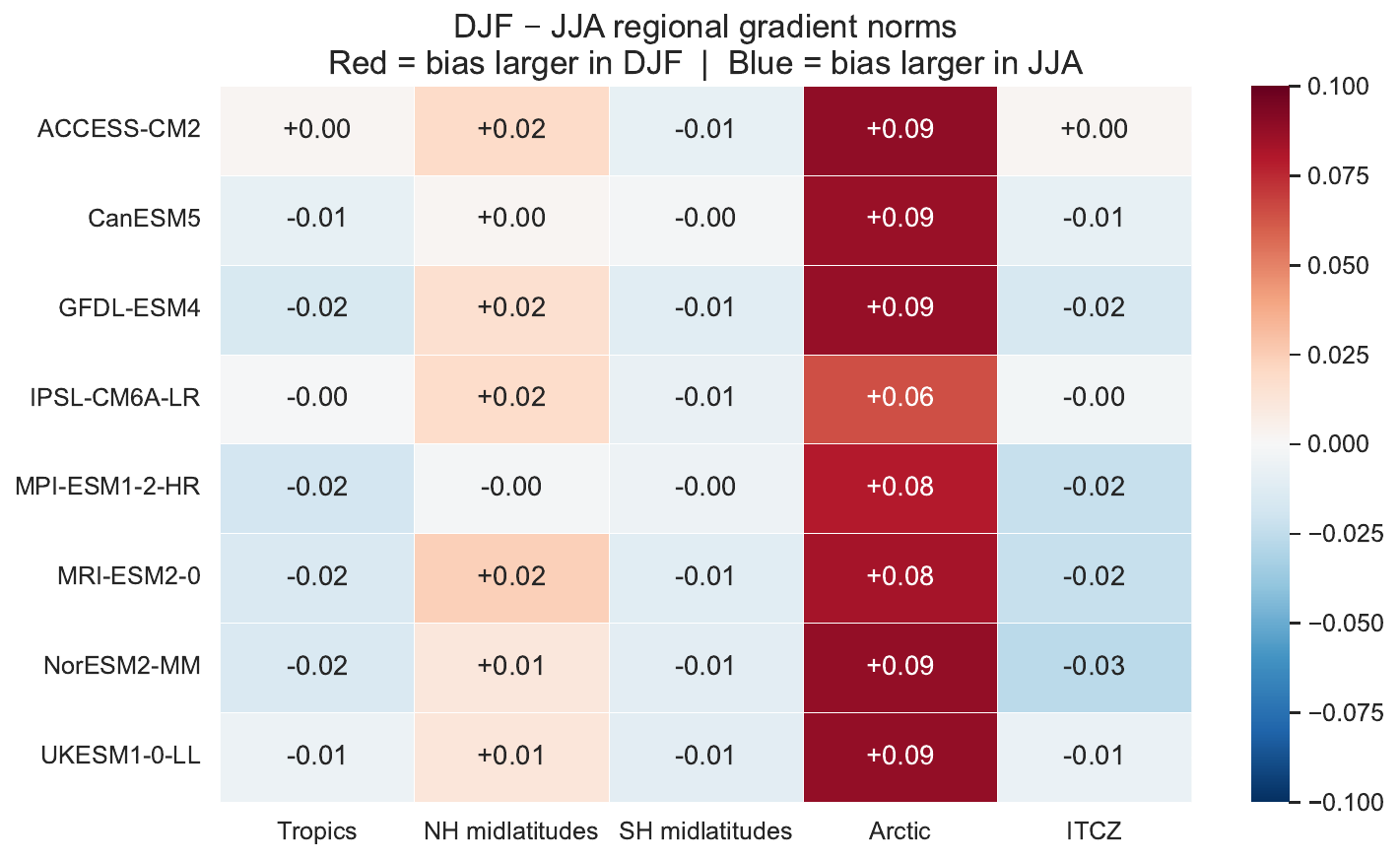}
  \caption{
    DJF$-$JJA difference of regional OT gradient norms under the
    unbalanced Sinkhorn divergence. Red: larger bias in DJF; blue: larger bias in JJA.
  }
  \label{fig:seasonal_diff}
\end{figure}

We partition $\mathbb{S}^2$ into five non-overlapping geographic
regions: Tropics ($|$lat$|<23.5^\circ$), ITCZ ($|$lat$|<10^\circ$), NH
Midlatitudes ($30^\circ$--$60^\circ$N), SH Midlatitudes ($30^\circ$--$60^\circ$S), and
Arctic ($>66.5^\circ$N), and compute the area-weighted RMS of $g_m$
over each region as a regional bias index. We provide in Appendix~\ref{appendix:iclimate},
Figure~\ref{fig:regional_gradients} these indices for all eight models and both seasons, normalized by the global maximum to preserve inter-model differences. We now precisely focus on the seasonal contrast ({\em i.e.} difference, denoted by $\Delta$ in the following) of these regional indices, which is not visible in the scalar rankings but reveals physically interpretable patterns of model bias and improvement. The Arctic column in Figure~\ref{fig:seasonal_diff} is the \emph{only} region with a meaningful seasonal signal,
showing consistently larger DJF bias across all eight models
($\Delta = +0.06$ to $+0.09$). All other regions show
differences of $|\Delta| \leq 0.02$, confirming that tropical,
midlatitude, and ITCZ precipitation biases are structurally
season-independent. The slight inter-model
variation in Arctic seasonal contrast (IPSL-CM6A-LR at $+0.06$
versus most others at $+0.09$) is physically interpretable as
reflecting differences in sea-ice parameterization complexity, as
discussed below.

\subsubsection{Arctic Bias and Sea-Ice Coupling}
\label{sec:arctic}

The seasonal collapse of the Arctic gradient norm, from a strongly
discriminating signal in DJF to near-zero in JJA across all models,
points to a specific physical mechanism: Arctic precipitation in
boreal winter is tightly coupled to sea-ice dynamics, through
moisture fluxes over the marginal ice zone, polar vortex
interactions, and storm track activity~\cite{serreze2019arctic}.
We investigate this coupling directly by overlaying the ERA5
climatological sea-ice edge (15\% concentration contour) on the
gradient maps of the two most contrasting models: ACCESS-CM2
(largest Arctic seasonal contrast, $\Delta = 0.089$) and
IPSL-CM6A-LR (smallest contrast, $\Delta = 0.065$). 
Figure~\ref{fig:arctic} shows the centered gradient $g_m$ over the Arctic in NorthPolarStereo projection for both models and both seasons, with the DJF and JJA sea-ice edges overlaid.

\begin{figure}[t]
  \centering
  \includegraphics[width=\linewidth]{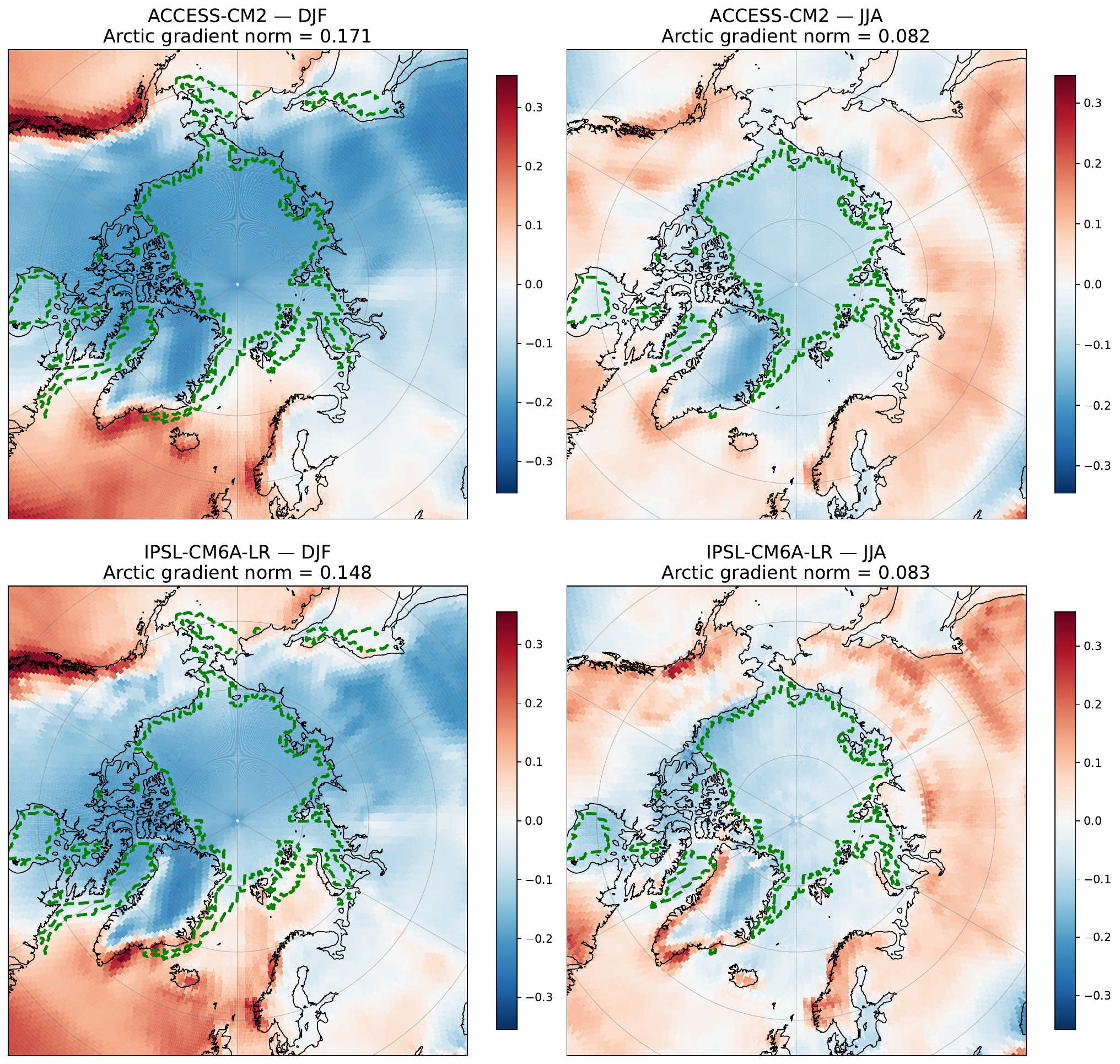}
  \caption{
    OT gradient of the Sinkhorn divergence (centered) over the Arctic
    for ACCESS-CM2 (top) and IPSL-CM6A-LR (bottom), in DJF (left)
    and JJA (right). Green contours show the ERA5 climatological
    sea-ice edge at 15\% concentration for DJF and JJA. Red: model has excess precipitation; blue: deficit.
    Arctic gradient norms are reported in the panel titles.
  }
  \label{fig:arctic}
\end{figure}

\paragraph{Results.}
Three findings stand out. First, in DJF the gradient signal is
\emph{spatially organized around the sea-ice edge} in both models: a
precipitation deficit appears inside the ice pack and an excess
appears along the marginal ice zone and just outside it. This dipole
structure is consistent with a systematic equatorward displacement or
overextension of the modeled ice edge, which shifts the
precipitation maximum outward relative to ERA5. The OT gradient
localizes this bias precisely, which a difference map cannot, since
the latter conflates the displacement with a local intensity error.

Second, the two models exhibit qualitatively different bias
geometries in DJF. ACCESS-CM2 shows a large-scale, coherent dipole
spanning the entire Arctic basin (deficit interior, excess
periphery), suggesting a bulk ice-edge displacement bias. IPSL-CM6A-LR
shows a more fragmented, spatially heterogeneous pattern concentrated
along the marginal ice zone, consistent with a fine-scale positional
error rather than a global shift. This distinction, directly readable
from the gradient geometry, is not recoverable from any scalar metric.

Third, in JJA both models converge to similar gradient patterns
(norms $0.082$ and $0.083$ respectively) as the ice retreats and its
organizing influence weakens. The residual JJA signal is
predominantly positive (model overestimates precipitation over the
open Arctic Ocean in summer), consistent with known warm biases in
Arctic sea-surface temperatures in coupled models. The near-identical
JJA patterns confirm that the ACCESS-CM2 vs.\ IPSL-CM6A-LR
difference is specifically a boreal winter phenomenon driven by
sea-ice representation, not a general Arctic precipitation
deficiency.

\paragraph{Summary.}
The experiments demonstrate that spherical unbalanced Sinkhorn
divergences can provide a principled, geometrically faithful, and
physically interpretable framework for climate model evaluation.
Beyond scalar rankings, the differentiability of the divergence
yields gradient maps that attribute model bias to specific geographic
regions and, in the Arctic, reveal a possible mechanistic link to sea-ice
edge geometry that is invisible to standard metrics.



\section{Discussion and Future Work}

In this work, we introduced \textsc{SHOT}, a novel algorithm for computing the Sinkhorn divergence associated to the heat kernel on the sphere. We demonstrated its efficiency and applicability to high-resolution data, and illustrated its potential for evaluating climate models in a geometry-aware manner. However, several limitations and avenues for future work remain.
First, as every entropy regularized OT methods, \textsc{SHOT} is subject to numerical instabilities for small $\varepsilon$, which limits its ability to approximate the true OT cost. To some extent, this can be amplified by
the spiking behavior of the heat kernel, which can produce negative values when truncated to a finite degree. Developing strategies to mitigate these instabilities and extend the operating range of the algorithm is an important direction for future work. One potential remedy is to combine the heat kernel with suitable filters to dampen its spiking behavior. We provide a more detailed discussion of these instabilities and potential solutions in Appendix~\ref{appendix:numerical_instabilites}. On the theoretical side, the framework relies on three main ingredients: Varadhan's formula, diagonalization of the heat operator on the manifold, and a fast associated transform. Extending this approach to other manifolds that share these properties is a natural direction for subsequent works. Finally, since the Sinkhorn divergence associated to the heat kernel is fully differentiable, it could serve as a geometry-aware training objective for neural networks operating on data that naturally lives on the sphere.
Lastly and regarding climate model evaluation, our goal was to illustrate the potential of the method rather than to provide a comprehensive assessment of CMIP6 models, which would require a broader analysis spanning multiple variables, seasons, and historical periods. Future work will explore these dimensions and investigate the use of the method for model development and tuning.

\bibliographystyle{plainnat}
\bibliography{example_paper.bib, climate.bib}

\appendix

\section{Proofs}
\subsection{Proof of Theorem~\ref{thm:conv-heat-reg}}\label{appendix:proof_convergence}
We will need some preliminaries. First, note that by Varadhan's celebrated short-time asymptotics of the heat kernel, one has uniform convergence of the time-$\varepsilon$ cost to the squared geodesic distance~\cite[Theorem 2.2]{varadhan1967} on the compact, connected Riemannian manifold $\M$:
\[
\widehat C_\varepsilon(x,y) = -\varepsilon \log \mathcal{H}_{\frac{\varepsilon}{4}}(x,y) \xrightarrow[\varepsilon \to 0]{\text{uniformly}} d(x,y)^2.
\]

Second, we will need some properties of the KL divergence, notably that it is jointly lower semi-continuous (see e.g.~\cite[Theorem 1]{posner2003random}) and that it is superlinear. For convenience of the reader, we state and reprove that latter statement.

\begin{lemma}[Superlinear lower bound for generalized KL]\label{lem:kl-mass}
Let $\nu \in \mathcal{R}_{+}(\M)$ with $\nu(\M) > 0$. Then for every $\mu \in \mathcal{R}_{+}(\M)$ with $\mu \ll \nu$,
\[
\mathrm{KL}(\mu \mid \nu) \;\geq\; \mu(\M)\,\log\frac{\mu(\M)}{\nu(\M)} - \mu(\M) + \nu(\M).
\]
In particular, the right-hand side is superlinear in $\mu(\M)$.
\end{lemma}

\begin{proof}
Applying Jensen's inequality to the convex function $t \mapsto t \log t$ with respect to the probability measure $\nu / b$ gives
\[
\mathrm{KL}(\mu \mid \nu) 
= \int \frac{\mathrm{d}\mu}{\mathrm{d}\nu} \log \frac{\mathrm{d}\mu}{\mathrm{d}\nu}\, \mathrm{d}\nu - \mu(\M) + \nu(\M)
\geq  \mu(\M) \log \frac{ \mu(\M)}{\nu(\M)} - \mu(\M) + \nu(\M). \qedhere
\]
\end{proof}

We will also use the fact that the problem $\mathrm{UOT}_{d^2}^{\tau}$ admits a minimizer. This is a consequence of~\cite[Theorem 3.3]{liero2018optimal} as the conditions of that theorem are satisfied: the problem is \emph{feasible} (take the null transport plan) and the KL divergence is superlinear (as stated just above). 

To the best of our knowledge, no full proof of a similar result exists for entropic regularized UOT: \cite[Theorem 3.3]{buze2023entropic} announce such a result for a future version of their paper but currently none is present while~\cite[Proposition 2.3]{nenna2025convergence} prove existence and uniqueness for a slightly more general context but restrict themselves to the discrete case, using arguments similar to those of the balanced case i.e. essentially the direct method of the calculus of variations and strict convexity of the regularizer.

As we will need such a result to study $\mathrm{UOT}^{\tau}_{\widehat C_{\varepsilon}, \varepsilon/4}$, before proceeding, we remark here that arguing along the same lines is sufficient to establish existence and uniqueness beyond the discrete case : observe first that by lower semicontinuity of the KL divergence, the functional defining regularized UOT is lower semicontinuous. Moreover, for $\varepsilon > 0$, superlinearity of the KL divergence implies that a minimizing sequence must be bounded in mass (see e.g.~\cite[Estimate 3.10]{liero2018optimal}) and thus, by Prokhorov's theorem for finite positive measures applied to the compact space $\M^2$, admits a converging subsequence in the topology of weak convergence of measures (this is essentially stating that the associated functional is \emph{coercive}) and existence of the minimizer follows. That the minimizer is unique derives from the strict convexity of the KL divergence regularizer. 

Finally, we need the following technical lemma (inspired by~\cite[Lemma~5.3]{nutz2021introduction}).

\begin{lemma}[Block approximation]\label{lem:block}
Let \(X\) be a compact metric space, let \(P,Q \in \mathcal{R}_{+}(X)\), and let
$\pi \in \mathcal{R}_{+}(X^2)$
be such that $\mathrm{KL}(\pi^1\mid P)<+\infty$ and $\mathrm{KL}(\pi^2\mid Q)<+\infty.$
Let moreover $f:X^2 \rightarrow \mathbb{R}$ be a continuous function. Then for every \(\eta>0\), there exists $\pi_\eta \in \mathcal{R}_{+}(X^2)$
such that $\mathrm{KL}(\pi_\eta\mid P\otimes Q)<+\infty$, having the same mass as $\pi$ and satisfying
\begin{enumerate}
\item $\mathrm{KL}(\pi_\eta^1\mid P) \le \mathrm{KL}(\pi^1\mid P)$ and $\mathrm{KL}(\pi_\eta^2\mid Q)
\le
\mathrm{KL}(\pi^2\mid Q)$,
\item $
\left|
\int_{X^2} f\, d\pi
-
\int_{X^2} f\, \mathrm{d}\pi_\eta
\right|
\le \eta\, m(\pi).$
\end{enumerate}

Moreover, if \(m(P)=m(Q)\) and \(\pi\) is a coupling between \(P\) and \(Q\), then \(\pi_\eta\) is also a coupling between \(P\) and \(Q\).

\end{lemma}

\begin{proof}
Let \(\eta>0\). Since $f$ is continuous and \(X^2\) is compact, \(f\) is uniformly continuous. Hence there exists \(\delta>0\) such that
$d(x,x')\le \delta, ~ d(y,y') \le \delta \implies 
|f(x,y)-f(x',y')|
\le \eta.$

By compactness of \(X\), there exists a finite measurable partition
$\{A_1,\dots,A_N\}$
of \(X\) such that
$
\operatorname{diam}(A_i)\le \delta$ for all $i$.
Define
$I^+
:=
\{(i,j): P(A_i)>0,\ Q(A_j)>0\}$ 
and set
\[
\pi_\eta
:=
\sum_{(i,j)\in I^+}
\frac{\pi(A_i\times A_j)}{P(A_i)Q(A_j)}
(P\otimes Q)|_{A_i\times A_j}.
\]
By construction $\pi_\eta \ll P \otimes Q$ and $\mathrm{KL}(\pi_\eta\mid P\otimes Q)<+\infty$ and moreover
\[
\frac{d\pi_\eta}{d(P\otimes Q)}
=
\sum_{(i,j)\in I^+}
\frac{\pi(A_i\times A_j)}{P(A_i)Q(A_j)}
\mathbf 1_{A_i\times A_j}.
\]
We now study the marginals. Since
$\mathrm{KL}(\pi^1\mid P)<+\infty$
we have \(\pi^1\ll P\), thus the Radon-Nikodym derivative $\frac{d\pi^1}{dP}$
exists and for every measurable \(B\subset X\), one has
\[
\pi_\eta^1(B) =
\pi_\eta(B\times X) =
\sum_{(i,j)\in I^+}
\frac{\pi(A_i\times A_j)}{P(A_i)Q(A_j)}
\,P(B\cap A_i)\,Q(A_j)=
\sum_i
\frac{\pi^1(A_i)}{P(A_i)}
\,P(B\cap A_i),
\]
as \(\{A_j\}_j\) is a partition of \(X\). Therefore,
\[
\pi_\eta^1
=
\sum_i
\frac{\pi^1(A_i)}{P(A_i)}
\,P|_{A_i},
\]
and thus
\[
\frac{d\pi_\eta^1}{dP}
=
\sum_i
\left(
\frac1{P(A_i)}
\int_{A_i}\frac{d\pi^1}{dP}\, \mathrm{d}P
\right)\mathbf 1_{A_i}.
\]
Recalling that
\[
\mathrm{KL}(\mu\mid \nu)
=
\int
\varphi\!\left(\frac{d\mu}{d\nu}\right)\mathrm{d}\nu
\text{ with }
\varphi(t):= t\log t - t +1,
\]
the convexity of \(\varphi\) and Jensen's inequality in normalized form yield
\[
\varphi\!\left(
\frac1{P(A_i)}
\int_{A_i}\frac{d\pi^1}{dP}\, \mathrm{d}P
\right)
\le
\frac1{P(A_i)}
\int_{A_i}
\varphi \left(\frac{d\pi^1}{dP}\right) \mathrm{d}P.
\]
Multiplying by \(P(A_i)\) and summing over \(i\), we obtain
\[
\mathrm{KL}(\pi_\eta^1\mid P)
=
\sum_i
P(A_i)
\varphi\!\left(
\frac1{P(A_i)}
\int_{A_i}\frac{d\pi^1}{dP}\, \mathrm{d}P
\right) 
\le
\sum_i
\int_{A_i}
\varphi \left(\frac{d\pi^1}{dP}\right)\, \mathrm{d}P 
=
\mathrm{KL}(\pi^1\mid P).
\]
This proves the first part of (1). The proof of the second part is identical.

We now prove that $\pi$ and $\pi_\eta$ have the same mass. Observe that if \((i,j)\notin I^+\), then
\[
P(A_i)=0
\qquad
\text{or}
\qquad
Q(A_j)=0.
\]
If \(P(A_i)=0\), then since \(\pi^1\ll P\),
\[
\pi(A_i\times A_j)
\le
\pi(A_i\times X)
=
\pi^1(A_i)
=
0,
\]
as \(\pi^1\ll P\).
Similarly, if \(Q(A_j)=0\), then $\pi(A_i\times A_j)=0$.
Consequently $
\pi(A_i\times A_j)=0$
$(i,j)\notin I^+$.
Therefore,
\[
m(\pi_\eta)
=
\sum_{(i,j)\in I^+}
\frac{\pi(A_i\times A_j)}{P(A_i)Q(A_j)}
(P\otimes Q)(A_i\times A_j) 
=
\sum_{(i,j)\in I^+}
\pi(A_i\times A_j) 
=
\sum_{i,j}
\pi(A_i\times A_j) 
=
m(\pi).
\]

Finally, we prove (2). From what precedes, as $\pi(A_i\times A_j)=0$ for all $(i,j)\notin I^+$, one has
\[
\int_{X^2} f\, \mathrm{d}\pi = \sum_{(i,j)\in I^+}
\int_{A_i\times A_j} \int f \mathrm{d}\pi.
\]
Define
\[
\bar f_{ij}
:=
\frac1{P(A_i)Q(A_j)}
\int_{A_i\times A_j}
f\, \mathrm{d}(P\otimes Q).
\]
Then
\[
\int_{X^2} f\, \mathrm{d}\pi_\eta
=
\sum_{(i,j)\in I^+}
\pi(A_i\times A_j)\bar f_{ij},
\]
hence
\[
\int_{X^2} f\, \mathrm{d}\pi
-
\int_{X^2} f\, \mathrm{d}\pi_\eta
=
\sum_{(i,j)\in I^+}
\int_{A_i\times A_j}
(f-\bar f_{ij})\, \mathrm{d}\pi.
\]

Fix \((i,j)\in I^+\) and \((x,y)\in A_i\times A_j\). Then
\[
\begin{aligned}
|f(x,y)-\bar f_{ij}|
&=
\left|
\frac1{P(A_i)Q(A_j)}
\int_{A_i\times A_j}
(f(x,y)-f(x',y'))
\,\mathrm{d}(P\otimes Q)(x',y')
\right| \\
&\le
\frac1{P(A_i)Q(A_j)}
\int_{A_i\times A_j}
|f(x,y)-f(x',y')|
\,\mathrm{d}(P\otimes Q)(x',y').
\end{aligned}
\]
Since the $A_i$ have diameter less than $\delta$ and $(x,y) \in A_i \times A_j$, the uniform continuity property mentioned at the beginning gives, for all $(x,y) \in A_i \times A_j, ~
|f(x,y)-\bar f_{ij}|
\le \eta.
$
Consequently,
\[
\left|
\int_{X^2} f\, \mathrm{d}\pi
-
\int_{X^2} f\, \mathrm{d}\pi_\eta
\right|
\le
\sum_{(i,j)\in I^+}
\int_{A_i\times A_j}
|f-\bar f_{ij}|\, \mathrm{d}\pi 
\le
\eta
\sum_{(i,j)\in I^+}
\pi(A_i\times A_j) 
=
\eta\, m(\pi).
\]

Assume now that \(m(P)=m(Q)\) and that \(\pi\) is a coupling between \(P\) and \(Q\), that is $\pi^1 = P$ and $\pi^2 = Q$. Then $\frac{\pi^1(A_i)}{P(A_i)}=1$
thus $\pi_\eta^1
=
\sum_i P|_{A_i}
=
P$,
since \(\{A_i\}\) is a partition of \(X\). Similarly, $
\pi_\eta^2=Q$ and therefore \(\pi_\eta\) is also a coupling between \(P\) and \(Q\).
\end{proof}

\begin{proof}[Proof of Theorem~\ref{thm:conv-heat-reg}]
We treat first the case $\tau < \infty$. The case $\tau = +\infty$ can be treated by similar arguments, with minor modifications, and is discussed at the end of the proof.
\paragraph{Case 1: $\tau < \infty$.}
On the space $\mathcal{R}_{+}(\M^2)$, consider the functional associated to the $\mathrm{UOT}^{\tau}_{d^2}$ problem:
\[
F(\pi) = \int_{\M^2} d^2 \mathrm{d}\pi + \tau \left( \mathrm{KL}(\pi^1 \mid P)
+ \mathrm{KL}(\pi^2 \mid Q)\right),
\]
and that associated to the heat kernel regularized $\mathrm{UOT}^{\tau}_{\widehat C_{\varepsilon}, \varepsilon/4}$, formally defined for $\varepsilon > 0$ by:
\[
\hat{F}_{\varepsilon}(\pi) = \int_{\M^2} \widehat C_{\varepsilon}\, \mathrm{d}\pi 
+ \tau\left( \mathrm{KL}(\pi^1 \mid P)
+ \mathrm{KL}(\pi^2 \mid Q)\right) + \frac{\varepsilon}{4} \,\mathrm{KL}(\pi \mid P \otimes Q).
\]
By construction, $\mathrm{UOT}^{\tau}_{d^2}(P,Q)$ is the infimum of $F$ while the heat-regularized cost $\mathrm{UOT}^{\tau}_{\widehat C_{\varepsilon}, \varepsilon/4}(P,Q)$ is the infimum value of $F_{\varepsilon}$. Let $\pi^* \in \mathcal{R}_{+}(\M^2)$ be a minimizer of $F^{\tau}$ \emph{i.e.} $F^{\tau}(\pi^*) = \mathrm{UOT}^{\tau}_{d^2}(P,Q)$.

The proof proceeds in two steps: a limsup inequality and a liminf inequality.

\textit{Step 1: Upper bound.} 
We show that
\[
\limsup_{\varepsilon \to 0}\; \mathrm{UOT}^{\tau}_{\widehat{C}_\varepsilon, \varepsilon/4}(P,Q) 
\le \mathrm{UOT}^{\tau}_{d^2}(P,Q) = F(\pi^*).
\]

Note that $\pi^*$ may have $\mathrm{KL}(\pi^* \mid P \otimes Q) = +\infty$; this is why we cannot simply take $\pi^*$ as a competitor for $F^{\tau}_\varepsilon$. Instead, we use the block approximation lemma~\ref{lem:block} with $X = \M, \pi = \pi^\star$ and $f=d^2$. Fixing $\eta > 0$, and letting $\pi_\eta$ be given by that lemma, one has by definition,
\begin{align*}
\mathrm{UOT}^{\tau}_{\widehat{C}_\varepsilon, \varepsilon/4}(P,Q) &\le F^{\tau}_\varepsilon(\pi_{\eta}) 
 \\
&\le \int \widehat{C}_\varepsilon\, \mathrm{d}\pi_{\eta} + \varepsilon \mathrm{KL}(\pi_{\eta}\mid P\otimes Q)
+ \tau\bigl(\mathrm{KL}((\pi^{*})^1\mid P) + \mathrm{KL}((\pi^{*})^2\mid Q)\bigr),
\end{align*}
where the inequality comes from item (1) of the Lemma \ref{lem:block}.
Write the difference in the transport cost:
\begin{align*}
\left|\int \widehat{C}_\varepsilon\, \mathrm{d}\pi_{\eta} - \int d^2\, \mathrm{d}\pi^*\right|
&\le \left|\int (\widehat{C}_\varepsilon - d^2)\, \mathrm{d}\pi_{\eta}\right| + \left|\int d^2\, \mathrm{d}\pi_{\eta} - \int d^2\, \mathrm{d}\pi^*\right| \\
&\le \|\widehat{C}_\varepsilon - d^2\|_\infty m(\pi_\eta) + \eta\, m(\pi^\star) = (\|\widehat{C}_\varepsilon - d^2\|_\infty + \eta) \, m(\pi^\star),
\end{align*}
by item (2) of the Lemma \ref{lem:block} and the fact that $\pi_\eta$ and $\pi^\star$ have the same mass. Consequently, one has, 
\begin{align*}
\mathrm{UOT}^{\tau}_{\widehat{C}_\varepsilon, \varepsilon/4}(P,Q)  \le & (\|\widehat{C}_\varepsilon - d^2\|_\infty + \eta) \, m(\pi^{\star}) \\
& \quad + \int d^2\, \mathrm{d}\pi^* + \varepsilon \mathrm{KL}(\pi_{\eta}\mid P\otimes Q)
+ \tau\bigl(\mathrm{KL}((\pi^{\star})^1\mid P) + \mathrm{KL}((\pi^{\star})^2\mid Q)\bigr). 
\end{align*}

As $\eta$ is fixed independently of $\varepsilon$, we can let $\varepsilon \to 0$ and get, by Varadhan's formula
\begin{align*}
\limsup_{\varepsilon \to 0}\; \mathrm{UOT}^{\tau}_{\widehat{C}_\varepsilon, \varepsilon/4}(P,Q) 
&\le \eta m(\pi^{\star}) + \int d^2\, \mathrm{d}\pi^* +\tau\bigl(\mathrm{KL}((\pi^{\star})^1\mid P) + \mathrm{KL}((\pi^{\star})^2\mid Q)\bigr) \\
&= \eta m(\pi^{\star}) + \mathrm{UOT}^{\tau}_{d^2}(P,Q).
\end{align*}

As $\eta > 0$ was arbitrary, we obtain the claimed upper bound.

\medskip
\textit{Step 2: Lower bound.}
We show that 
\[
\mathrm{UOT}^{\tau}_{d^2}(P,Q) = F(\pi^*)
\le
\liminf_{\varepsilon\to0}
\mathrm{UOT}^{\tau}_{\widehat C_{\varepsilon},\varepsilon/4}(P,Q).
\]

For each \(\varepsilon>0\), let \(\pi_\varepsilon\in\mathcal R_+(\M^2)\) be a minimizer of \(\hat F_\varepsilon\).

From the first step, one sees that the limsup of the family $\{\mathrm{UOT}^{\tau}_{\widehat{C}_\varepsilon, \varepsilon/4}(P,Q)\}_{\varepsilon}$ is finite and as a consequence, that family is bounded, say by $\kappa > 0$. 

By Varadhan's uniform convergence $\widehat C_\varepsilon\to d^2\ge 0$ on the compact manifold $\M$, there exists $\varepsilon_{0}>0$ such that $\widehat C_\varepsilon\ge -1$ for all $\varepsilon\le\varepsilon_0$. Since $\mathrm{KL}(\pi_\varepsilon^2\mid Q)\ge 0$ and $\mathrm{KL}(\pi_\varepsilon\mid P\otimes Q)\ge 0$, we may drop these terms and apply Lemma~\ref{lem:kl-mass} to the first marginal to obtain
\[
\tau\Bigl(m(\pi_\varepsilon^1)\log\frac{m(\pi_\varepsilon^1)}{b} - m(\pi_\varepsilon^1) + b\Bigr)
\;\le\;
\tau\,\mathrm{KL}(\pi_\varepsilon^1\mid P)
\;\le\;
\kappa + \,m(\pi_\varepsilon^1).
\]
Since the left-hand side is superlinear and the right-hand side is affine in $m(\pi_\varepsilon^1)$, we deduce that $m(\pi_\varepsilon)=m(\pi_\varepsilon^1)$ is bounded uniformly for $\varepsilon \le \varepsilon_{0}$ i.e. the family $(\pi_{\varepsilon})_{\varepsilon \le \varepsilon_{0}}$ is bounded in mass.

By definition of the \(\liminf\), there exists a sequence \((\varepsilon_k)_k\) converging to $0$ such that 
\[
\mathrm{UOT}^{\tau}_{\widehat{C}_k,\varepsilon_k/4}(P,Q)
\longrightarrow
\liminf_{\varepsilon\to0}
\mathrm{UOT}^{\tau}_{\widehat C_{\varepsilon},\varepsilon/4}(P,Q).\]
We let $k_0$ be such that $\varepsilon_{k}\leq \varepsilon_0$ for $k\geq k_0$ and henceforth work with the sequence $(\pi_{k})_{k \geq k_0}$. To simplify notation, we write $\widehat{C}_{\varepsilon_k}:=\widehat{C}_k, F_{\varepsilon_{k}} = F_k$ and $\pi_k:=\pi_{\varepsilon_k}$.

We now apply Prokhorov's theorem on $\mathcal{R}_{+}(\M^2)$ : as $\M^2$ is compact and $(\pi_{k})_{k \ge k_{0}}$ is bounded in mass, the sequence $(\pi_{k})_{k \geq k_0}$ is relatively compact in $\mathcal{R}_{+}(\M^2)$. Extract a subsequence $\pi_{k_n}$ that converges weakly to some $\pi\in\mathcal{R}_+(\M^2)$.

By definition of $\pi^{\star}$ we have
\[
\mathrm{UOT}^\tau_{d^2}(P,Q) = F(\pi^{\star}) \le F(\pi) = \int_{\M^2} d^2 \mathrm{d}\pi + \tau \left(\mathrm{KL}(\pi^1 \mid P)
+ \mathrm{KL}(\pi^2 \mid Q)\right).
\]
Now, by lower semicontinuity of the KL divergence and continuity of the marginal-taking operation
\[
\mathrm{KL}(\pi^1 \mid P) = \mathrm{KL}((\lim_{n} \pi_{k_{n}})^{1} \mid P) = \mathrm{KL}(\lim_{n} \pi_{k_{n}}^{1} \mid P) \le \lim_{n} \mathrm{KL}(\pi_{k_{n}}^1 \mid P)
\]
and similarly for the second marginal. Moreover, as the KL divergence is nonnegative, we get
\[
F(\pi) \le \int_{\M^2} d^2 \mathrm{d}\pi + \lim_{n} \Bigg( \tau \left(\mathrm{KL}(\pi_{k_{n}}^1 \mid P)
+ \mathrm{KL}(\pi_{k_{n}}^2 \mid Q)\right) + \frac{\varepsilon_{k_n}}{4} \,\mathrm{KL}(\pi_{k_{n}} \mid P \otimes Q) \Bigg).
\]
It remains to examine the first term. Write
\[
\int_{\M^2}  \widehat{C}_{k_n}\, \mathrm{d}\pi_{k_n}
= \int_{\M^2} d^2\, \mathrm{d}\pi_{k_n} + \int_{\M^2} ( \widehat{C}_{k_n} - d^2)\, \mathrm{d}\pi_{k_n}.
\]
By the uniform convergence of $\widehat{C}_k$ to $d^2$ and the uniform bound on $m(\pi_{k_n})$, as $n \to +\infty$,
\[
\biggl|\int_{\M^2}(\widehat{C}_{k_n} - d^2)\, \mathrm{d}\pi_{k_n}\biggr|
\le \|\widehat{C}_{k_n} - d^2\|_\infty \cdot m(\pi_{k_n}) \longrightarrow 0.
\]
Since $d^2$ is continuous on the compact $\M^2$ and $\pi_{k_n} \rightharpoonup \pi$, the weak convergence gives
\[
\int_{\M^2} d^2\, \mathrm{d}\pi_{k_n} \longrightarrow \int_{M^2} d^2\, \mathrm{d}\pi.
\]
Hence,
\[
\lim_{n} \int_{\M^2} \widehat{C}_{k_n}\, \mathrm{d}\pi_{k_n} = \int_{\M^2} \mathrm{d}^2\, \mathrm{d}\pi.
\]
All in all, we get that
\[
F(\pi) \le \lim_{n} \Bigg (\int_{\M^2} \widehat{C}_{k_n}\, \mathrm{d}\pi_{k_n} + \tau \left(\mathrm{KL}(\pi_{k_{n}}^1 \mid P)
+ \mathrm{KL}(\pi_{k_{n}}^2 \mid Q)\right) + \frac{\varepsilon_{k_n}}{4} \,\mathrm{KL}(\pi_{k_{n}} \mid P \otimes Q) \Bigg).
\]

As $\varepsilon_{k}$ was chosen to realize the $\liminf$ of $F_{\varepsilon}(\pi_{\varepsilon})$, a property inherited by its subsequence $\varepsilon_{k_{n}}$, we get
\[
\mathrm{UOT}^\tau_{d^2}(P,Q) \le F(\pi) \le \lim_{n \to \infty} F_{k_{n}}(\pi_{k_n}) = \liminf_{\varepsilon \to 0} F_{\varepsilon}(\pi_{\varepsilon}) = \liminf_{\varepsilon\to0}
\mathrm{UOT}^\tau_{\widehat C_{\varepsilon},\varepsilon/4}(P,Q),
\]
and the lower bound follows.

We are now able to conclude the proof, by combining the lower and upper bounds
\[
\mathrm{UOT}^\tau_{d^2}(P,Q)
\le \liminf_{\varepsilon\to0}
\mathrm{UOT}^\tau_{\widehat C_{\varepsilon},\varepsilon/4}(P,Q) \le \limsup_{\varepsilon\to0}
\mathrm{UOT}^\tau_{\widehat C_{\varepsilon},\varepsilon/4}(P,Q) \le 
\mathrm{UOT}^\tau_{d^2}(P,Q).
\]
Equality must hold everywhere so $\lim_{\varepsilon\to0}
\mathrm{UOT}^{\tau}_{\widehat C_{\varepsilon},\varepsilon/4}(P,Q)$ exists and is equal to $
\mathrm{UOT}^{\tau}_{d^2}(P,Q)$.

\paragraph{Case 2: $\tau = + \infty$.}
In the balanced case where $P$ and $Q$ are probability measures and we impose $\pi^1 = P$, $\pi^2 = Q$ (formally $\tau = \infty$), the marginal KL terms force the marginals to match exactly. The main adaptation from the unbalanced case is to make sure that all the necessary operations can be done in the set of couplings $\U(P,Q)$. This is a consequence of the following remarks :
\begin{enumerate}
	\item As $\M$ is compact and $d^2$ is continuous, an optimal coupling exists for the balanced OT problem~\cite[Thm.~1.4]{santambrogio2015optimal}. Hence $\pi^{\star}$ is in fact a coupling.
	\item For the upper bound, by the block approximation Lemma~\ref{lem:block}, we get that the $\pi_{\eta}$ are couplings, which ensures that the proof goes through.
	\item For the lower bound, one observes that regularized entropic OT admits a minimizer and that it is a coupling (see, e.g., \cite[Theorem 4.2]{nutz2021introduction}) i.e. $(\pi_{\varepsilon})_{\varepsilon > 0}$ is in fact a family of couplings. As a result, so are the sequence $(\pi_{k})_{k \ge k_{0}}$ and its subsequence $(\pi_{k_{n}})_{n}$ as well as the limit $\pi$ of that subsequence. This ensures that the proof goes through.
\end{enumerate}
\end{proof}

\subsection{Proof of Proposition~\ref{prop:sinkh:cond}}\label{app:proof:sinkh-cond}

We will in fact show that the claimed result holds for any compact, connected symmetric space. This implies as a corollary that it holds for the sphere $\S^2$, since the latter is a symmetric space \cite{helgason1979differential, eschenburg1997lecture}.

\begin{proposition}[Sinkhorn divergence conditions]\label{prop:sinkh:cond:restate}
Let $\M$ be a compact, connected symmetric space. Then:
\begin{enumerate}[label=(\roman*)]
\item $\M$ is compact;
\item the time-$\varepsilon$ cost function $\widehat C_\varepsilon(x,y) := -\varepsilon \log\mathcal{H}_{\varepsilon/4}(x,y)$ is Lipschitz continuous with respect to the geodesic distance $d$;
\item the kernel $K_\varepsilon(x,y) = e^{-\widehat C_\varepsilon(x,y)/\varepsilon} = \mathcal{H}_{\varepsilon/4}(x,y)$ is positive definite and universal.
\end{enumerate}
\end{proposition}

\begin{proof}
(i) Compactness is assumed in the hypothesis.

(ii) The heat kernel $\mathcal{H}_\varepsilon$ on a connected manifold is smooth and strictly positive for $\varepsilon>0$ \cite[Thm.~7.20, Cor.~8.12]{grigoryan2009heat}. Therefore $\widehat C_\varepsilon$ is smooth on $\M^2$. By compactness of $\M$, its Riemannian gradient is globally bounded, implying that $\widehat C_\varepsilon$ is Lipschitz continuous with respect to the geodesic distance (see \cite[Prop.~10.43]{boumal2023introduction}). 

(iii) Positive definiteness and universality follow from results in \cite{steinert2025universal}. More precisely, therein, kernels on symmetric spaces are defined and studied through their associated \emph{spectral densities}. There, Lemma~1 establishes that non-negativity of the spectral density implies that the associated kernel is positive definite, which holds for the heat kernel by way of an explicit formula available for the spectral density. Moreover, Theorem~3 gives conditions for which a kernel is $\mathcal{C}_{c}$-universal (i.e. local uniform approximation), as discussed in \cite{micchelli2006universal} and required in~\cite{pmlr-v89-feydy19a}; moreover, it is explicitly stated that this applies to the heat kernel. This is precisely what we need for our purposes, concluding the proof.
\end{proof}

In particular, taking $\M = \S^2$, all three conditions are satisfied, so that by \cite[Theorem 5 \& 6]{sejourne2019sinkhorn}, the associated Sinkhorn divergence $\hat{S}^\tau_\varepsilon$ defines a symmetric, positive definite, smooth function that is convex in each of its input variables and metrizes the convergence in law.

\subsection{Proof of Proposition~\ref{prop:kern-to-fourier}}\label{app:proof:kern-to-fourier}
Our integration convention follows common practice in the computational spherical-harmonics literature: integrals are taken without normalization (e.g., $SO(3)$ has volume $8\pi^{2}$). For a comprehensive treatment, see~\cite{driscoll1994computing,faraut2008analysis}.
\begin{proof}
Fix $x \in \S^2$ and recall that $\vec{n} = (0,0,1)^\top$ is the north pole. Using the fact that averaging over $\psi \in [0,2\pi)$ does not change the value, we have
\[
Kf(x) = \frac{1}{2\pi}\int_{0}^{2\pi} Kf(x) \, d\psi
= \frac{1}{2\pi} \int_{0}^{2\pi} \int_{\S^2} K(x,y) f(y) \, \sigma(dy) \, d\psi,
\]
where $\sigma$ denotes the surface measure on $\S^2$.
We now parametrize the sphere via the Euler-angle representation of rotations. Let $R(\phi,\theta,\psi) \in \mathrm{SO}(3)$ denote a rotation with Euler angles $(\phi,\theta,\psi)$; then the map
\begin{equation*}	
R(\phi,\theta,\psi) \mapsto y(\phi,\theta) = R(\phi,\theta,\psi)\vec{n} \in \S^2
\end{equation*}
covers $\S^2$ exactly once as $\phi$ and $\theta$ vary while $\psi$ is held fixed. Under this parametrization, the surface element satisfies $\sigma(dy) = \sin\theta \, d\phi \, d\theta$. Applying Fubini's theorem, one gets
\[
\int_{0}^{2\pi} \int_{\S^2} K(x,y) f(y) \, \sigma(dy) \, d\psi
= \int_{0}^{2\pi}\int_{0}^{\pi}\int_{0}^{2\pi} 
K\bigl(x,R(\phi,\theta,\psi)\vec{n}\bigr) f\bigl(R(\phi,\theta,\psi)\vec{n}\bigr) 
\sin\theta \, d\phi \, d\theta \, d\psi.
\]
The right-hand side is precisely the integral over $\mathrm{SO}(3)$ with respect to the (unnormalized) volume measure in Euler coordinates thus
\begin{equation}\label{eq:kernf1}
Kf(x) = \frac{1}{2\pi}\int_{\mathrm{SO}(3)} K(x,R\vec{n}) f(R\vec{n}) \, dR.
\end{equation}
Now, using the radial symmetry of the kernel $K(x,y) = K_0(\langle x,y\rangle)$, one has
\[
K(x,R\vec{n}) = K_0\bigl(\langle x, R\vec{n}\rangle\bigr) 
= K_0\bigl(\langle R^{-1}x, \vec{n}\rangle\bigr)
= K(R^{-1}x, \vec{n}),
\]
so that from equation~\eqref{eq:kernf1}, one gets
\[
Kf(x) = \frac{1}{2\pi}\int_{\mathrm{SO}(3)} K(x,R\vec{n}) f(R\vec{n}) \, dR
= \frac{1}{2\pi}\int_{\mathrm{SO}(3)} f(R\vec{n}) \, K(R^{-1}x, \vec{n}) \, dR
= \frac{1}{2\pi}(f \star \tilde K)(x),
\]
with $\tilde K(z) = K(z, \vec{n})$, concluding the proof.
\end{proof}

\section{Spherical Harmonics}\label{app:spherical-harmonics}
\paragraph{Coordinates on $\S^2$.}

Points on the two-sphere $\S^2 = \{x \in \mathbb{R}^3 : \|x\| = 1\}$ are parametrized by colatitude $\theta \in [0,\pi]$ and longitude $\varphi \in [0, 2\pi)$, with the identification
\[
x = \bigl(\cos\varphi\sin\theta,\; \sin\varphi\sin\theta,\; \cos\theta\bigr)^\top \in \mathbb{R}^3.
\]

\paragraph{The rotation group $\mathrm{SO}(3)$.}
The symmetry group of $\S^2$ is the \emph{special orthogonal group}
\[
\mathrm{SO}(3) = \bigl\{ R \in \mathbb{R}^{3\times 3} : R^\top R = I,\; \det R = 1 \bigr\},
\]
consisting of all proper (orientation-preserving) rotations of $\mathbb{R}^3$.
Any $R \in \mathrm{SO}(3)$ can be factored using the \emph{Euler angles} $\varphi \in [0,2\pi)$, $\theta \in [0,\pi]$, $\psi \in [0,2\pi)$ as
\[
R(\phi, \theta, \psi) = R_z(\phi)\, R_y(\theta)\, R_z(\psi),
\]
where $R_z$ and $R_y$ denote elementary rotations around the $z$- and $y$-axes:
\[
R_z(\alpha) =
\begin{pmatrix}
\cos\alpha & -\sin\alpha & 0 \\
\sin\alpha &  \cos\alpha & 0 \\
0 & 0 & 1
\end{pmatrix},
\qquad
R_y(\alpha) =
\begin{pmatrix}
 \cos\alpha & 0 & \sin\alpha \\
 0 & 1 & 0 \\
-\sin\alpha & 0 & \cos\alpha
\end{pmatrix}.
\]
Unlike translations in $\mathbb{R}^n$, rotations do not commute in general, making $\mathrm{SO}(3)$ non-abelian.

\paragraph{$L^2$ inner product.}
The natural inner product on $\S^2$ is
\[
\langle u, v \rangle_{L^2(\S^2)}
= \int_{\S^2} u\, \bar{v}\; \mathrm{d}\Omega
= \int_0^{2\pi}\!\int_0^{\pi}
  u(\theta,\varphi)\,\overline{v(\theta,\varphi)}\;\sin\theta\,\mathrm{d}\theta\,\mathrm{d}\varphi,
\]
where $\mathrm{d}\Omega = \sin\theta\,\mathrm{d}\theta\,\mathrm{d}\varphi$ is the rotation-invariant surface measure on $\S^2$.
This defines the Hilbert space $L^2(\S^2)$ of square-integrable functions on the sphere.

\paragraph{Spherical harmonics.}
The \emph{spherical harmonics} are defined for degree $\ell \ge 0$ and order $|m| \le \ell$ by
\[
Y_{\ell m}(\theta,\varphi)
= N_{\ell m}\, P_\ell^m(\cos\theta)\, e^{im\varphi},
\qquad
N_{\ell m} = \sqrt{\frac{2\ell+1}{4\pi}\frac{(\ell-m)!}{(\ell+m)!}},
\]
where $P_\ell^m$ are the associated Legendre polynomials.
The normalization $N_{\ell m}$ is chosen so that the family $\{Y_{\ell m}\}_{\ell \ge 0,\,|m|\le\ell}$ is orthonormal:
\[
\langle Y_{\ell m},\, Y_{\ell' m'} \rangle_{L^2(\S^2)} = \delta_{\ell\ell'}\,\delta_{mm'}.
\]
Among all possible orthonormal bases of $L^2(\S^2)$, the spherical harmonics are distinguished by their equivariance under $\mathrm{SO}(3)$: applying any rotation $R \in \mathrm{SO}(3)$ to $Y_{\ell m}$ produces a linear combination of harmonics of the \emph{same} degree $\ell$,
\[
Y_{\ell m}(R^{-1}\,\cdot) = \sum_{|m'|\le\ell} D^\ell_{m'm}(R)\, Y_{\ell m'},
\]
where $D^\ell(R)$ is the Wigner $D$-matrix of degree $\ell$.
The family $\{Y_{\ell m}\}$ is moreover complete, so any $f \in L^2(\S^2)$ admits the \emph{spherical harmonic expansion}
\[
f(\theta,\varphi)
= \sum_{\ell=0}^{\infty}\sum_{m=-\ell}^{\ell}
  \widehat{f}_{\ell m}\, Y_{\ell m}(\theta,\varphi),
\qquad
\widehat{f}_{\ell m} = \mathcal{F}[f](\ell,m)
\coloneqq \langle f,\, Y_{\ell m}\rangle_{L^2(\S^2)}.
\]

\paragraph{Spherical convolution.}
On $\S^2$, convolution is defined by averaging over rotations rather than translations. Following \cite{driscoll1994computing}, this leads to the spherical convolution operator
\[
(f \star \kappa)(x) = \int_{\mathrm{SO}(3)} f(R\vec{n}) \, \kappa(R^{-1}x) \, \mathrm{d}R,
\]
which admits a convolution theorem in the spherical harmonic domain.

Proposition~\ref{prop:kern-to-fourier} shows that our radial kernel operator coincides exactly with spherical convolution (up to normalization), enabling fast Sinkhorn iterations via fast SHTs. Other uses of spherical convolution for neural networks include \cite{kondor2018generalization} which studies group convolutions on homogeneous spaces including $\S^2$, while \cite{sosnovik2021disco} develop scale-equivariant layers via hierarchical sampling.

\section{Implementation Details}\label{appendix:implementation_details}
Our implementation uses the differentiable spherical harmonic transform libraries \texttt{torch-harmonics}~\cite{bonev2023spherical} and \texttt{s2fft}~\cite{price2024differentiable}, built on top of \texttt{PyTorch}~\cite{paszke2019pytorch} and \texttt{JAX}~\cite{jax2018github}.

\begin{algorithm}[H]
\caption{Spherical Harmonic Sinkhorn}
\label{alg:sinkhorn}
\begin{algorithmic}[1]
\REQUIRE $\vector{p}, \vector{q} \in \mathbb{R}^n_+$ (source and target marginals), $\varepsilon > 0$ (regularization), $\tau \in (0, +\infty]$ (unbalanced parameter), $T \in \mathbb{N}$ (max iterations), $L \in \mathbb{N}$ (band-limit), $\lbrace a_i\rbrace$ (area weights of the discretization), $\delta > 0$ (convergence threshold), $C \in \mathbb{N}$ (check frequency)
\STATE $t \leftarrow \varepsilon/4$
\STATE $h_\ell \leftarrow e^{-t\ell(\ell+1)},\quad \ell = 0, \ldots, L-1$
\STATE $\vector{u} \leftarrow \mathbf{1}$, $\vector{v} \leftarrow \mathbf{1}$
\FOR{$k = 1, \ldots, T$}
    \IF{$\tau = +\infty$}
        \STATE $\vector{u} \leftarrow \vector{p} \oslash \mathrm{HeatConv}(\vector{v};\,\{h_\ell\})$            \hfill\COMMENT{with $\oslash$ being the element-wise division}
        \STATE $\vector{v} \leftarrow \vector{q} \oslash \mathrm{HeatConv}(\vector{u};\,\{h_\ell\})$
        \STATE $s \leftarrow \textstyle\sum_i u_i$ 
        \STATE $\vector{u} \leftarrow \vector{u}/s$,\quad $\vector{v} \leftarrow s\,\vector{v}$
    \ELSE
        \STATE $\vector{u} \leftarrow \bigl(\vector{p} \oslash \mathrm{HeatConv}(\vector{v};\,\{h_\ell\})\bigr)^{\tau/(\tau+\varepsilon)}$
        \STATE $\vector{v} \leftarrow \bigl(\vector{q} \oslash \mathrm{HeatConv}(\vector{u};\,\{h_\ell\})\bigr)^{\tau/(\tau+\varepsilon)}$
        \STATE $c \leftarrow \dfrac{\tau}{2}\left(\log\left(\sum_i a_i\,p_i\,u_i^{-\varepsilon/\tau}\right) - \log \left(\sum_i a_i\,q_i\,v_i^{-\varepsilon/\tau}\right)\right)$
        \STATE $\vector{u} \leftarrow e^{c/\varepsilon}\,\vector{u}$,\quad $\vector{v} \leftarrow e^{-c/\varepsilon}\,\vector{v}$
    \ENDIF
    \IF{$k \bmod C = 0$}
        \STATE $e_{\mathrm{marg}} \leftarrow \bigl\|\vector{u} \odot \mathrm{HeatConv}(\vector{v};\,\{h_\ell\}) - \vector{p}\bigr\|_2 \,/\, \|\vector{p}\|_2$
        \IF{$e_{\mathrm{marg}} < \delta$}
            \STATE \textbf{break}
        \ENDIF
    \ENDIF
\ENDFOR
\STATE $\vector{\phi} \leftarrow \varepsilon \log \vector{u}$,\quad $\vector{\psi} \leftarrow \varepsilon \log \vector{v}$
\IF{$\tau = +\infty$}
    \STATE \textbf{return} $\sum_i a_i \bigl[\phi_i\,p_i + \psi_i\,q_i\bigr]$
\ELSE
    \STATE \textbf{return} $\sum_i a_i \bigl[\tau(1 - e^{-\phi_i/\tau})\,p_i + \tau(1 - e^{-\psi_i/\tau})\,q_i\bigr]$
\ENDIF
\end{algorithmic}
\end{algorithm}

\subsection{Discretizations of $\S^2$}\label{appendix:discretization_sphere}
In practice, functions on $\mathbb{S}^2$
are represented on a finite grid of sampling points $(\theta_i, \phi_j)$. We briefly review several such discretization schemes, illustrated in Figures \ref{fig:discretization_th} and \ref{fig:discretization_hp}. Certain schemes admit an exact sampling theorem: for bandlimited signals, i.e., functions whose expansion in the spherical harmonics is supported on degrees $\ell \leq L$, a sufficiently dense sampling guarantees that both the forward and inverse Spherical Harmonic Transforms are exact.

\begin{itemize}
	\item \textbf{Equiangular (Driscoll-Healy \cite{driscoll1994computing} / McEwen-Wiaux \cite{mcewen2011novel} / Clenshaw-Curtis \cite{clenshaw1960}).} Samples are placed on a regular grid equally spaced in longitude and at equiangular latitudes. The resolution is parametrized by a band-limit $L$. These grids admit exact sampling theorems, which require $\mathcal{O}(L^2)$ samples.
    
	\item \textbf{Gauss-Legendre.} Latitudinal samples are placed at the $L$ roots of the Legendre polynomial $P_L$, yielding an exact sampling theorem. As with equiangular grids, the resolution is parametrized by $L$ and the total number of samples is $\mathcal{O}(L^2)$.
    
	\item \textbf{Lobatto (Clenshaw-Curtis with poles).} A variant of equiangular sampling that includes the poles as quadrature nodes, using the zeros of $(1-x^2)P'_{L-1}(x)$ in the latitudinal direction. It shares the same band-limit parametrization and sampling requirements as the schemes above.
    
	\item \textbf{HEALPix \cite{gorski2005healpix}.} The sphere is partitioned into $N = 12 N_\mathrm{side}^2$ equal-area pixels arranged in iso-latitude rings. The resolution is controlled by $N_\mathrm{side}$. No exact sampling theorem is available, and $N$ is an independent resolution parameter.
\end{itemize}

For exact iso-latitudinal schemes (Clenshaw-Curtis, Gauss-Legendre, Lobatto), spherical harmonic transforms exploit a separation-of-variables structure. A direct implementation scales as $\mathcal{O}(L^4)$, while the standard semi-naïve algorithm reduces this to $\mathcal{O}(L^3)$ by combining FFTs in longitude with projections onto associated Legendre polynomials. Faster algorithms with complexity $\mathcal{O}(L^2 \log^2 L)$, based on fast polynomial transforms, have been proposed~\cite{healy2003ffts} and in some cases extend beyond specific sampling schemes~\cite{kunis2003fast}. However, these approaches typically require additional stabilization procedures whose numerical robustness and complexity are not fully understood in general. As a result, despite their favorable asymptotic scaling, they are not yet as reliable in practice. Consequently, most implementations (e.g., \cite{price2024differentiable}) rely on the semi-naïve $\mathcal{O}(L^3)$ algorithms for both forward and inverse transforms, which offer a more robust trade-off between efficiency and numerical stability. Similar considerations apply to HEALPix, where fast methods exist~\cite{kunis2003fast, drake2020fast}  but practical computations often favor stable, moderately super-quadratic algorithms.

In our experiments, we use the Clenshaw-Curtis grid for an equiangular sampling and the HEALPix grid for an equal-area sampling scheme as it is widely used in the climate science community. However our implementation is flexible and can be easily adapted to other discretization schemes.

\begin{figure}[h]
  \centering
  \includegraphics[width=\linewidth]{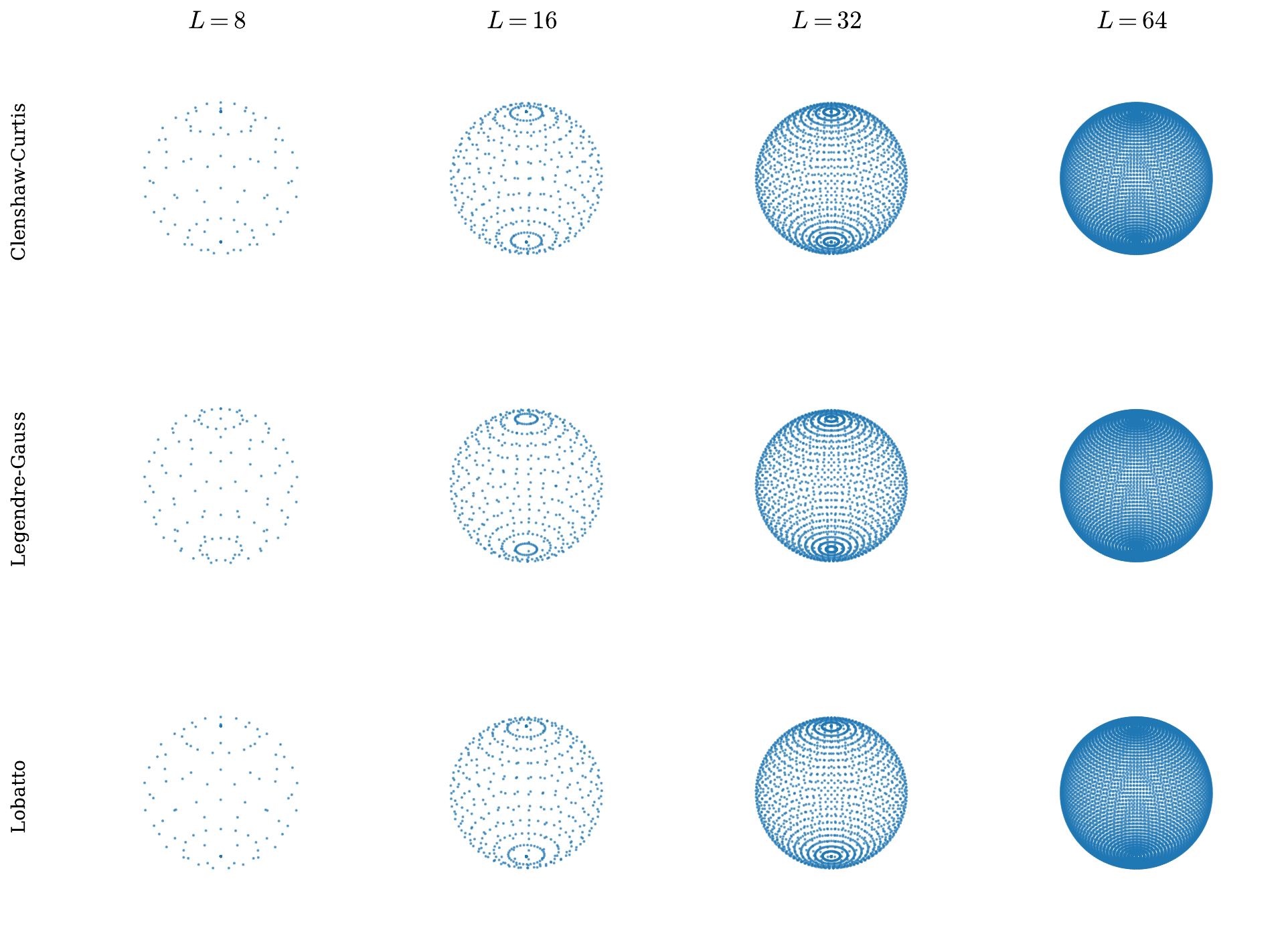}
  \caption{Clenshaw-Curtis, Gauss-Legendre and Lobbato discretization schemes for varying $L$.}
  \label{fig:discretization_th}
\end{figure}

\begin{figure}[h]
  \centering
  \includegraphics[width=\linewidth]{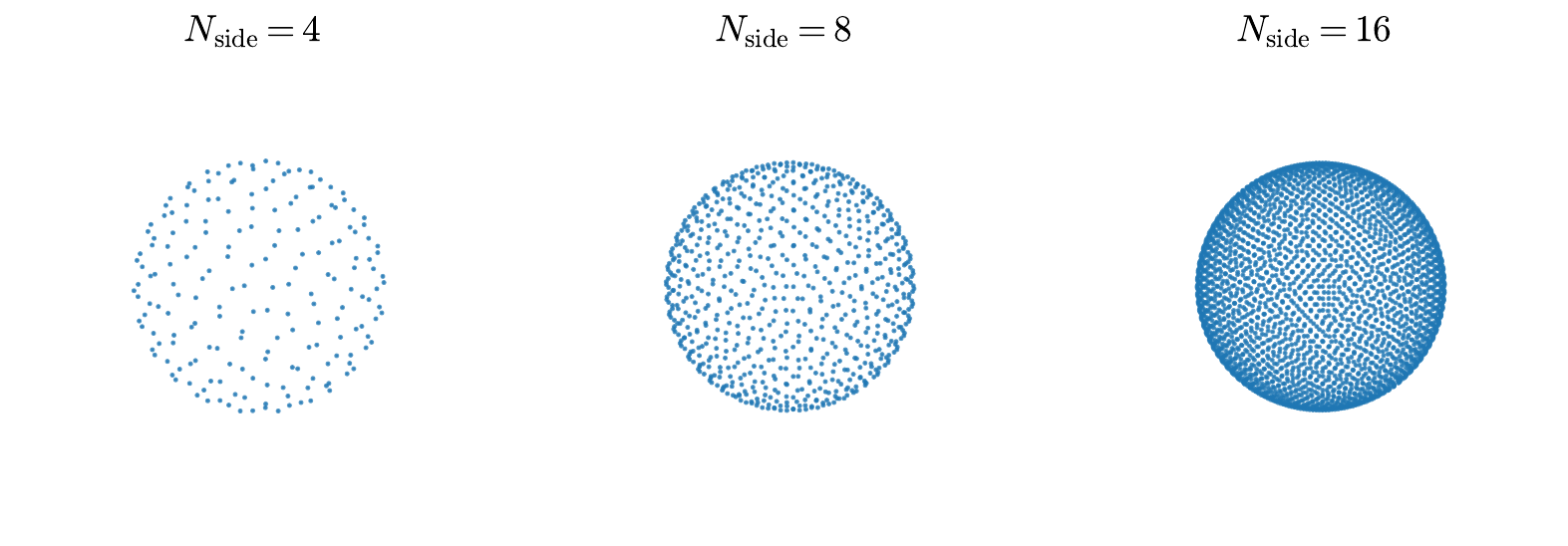}
  \caption{HEALPix discretization scheme for varying $N_\mathrm{side}$.}
  \label{fig:discretization_hp}
\end{figure}

\subsection{Discrete OT on the Sphere}
\label{appendix:discretization_ot}

Let $\{x_i\}_{i=1}^n$ be a discretization of $\S^2$ with area weights
$\vector{a} \in \mathbb{R}^n_+$, $\vector{a}^\top \mathbf{1} = 4\pi$, so that
$\int_{\S^2} p\,\mathrm{d}\sigma \approx \vector{a}^\top \vector{p}$ where $\vector{p}=(p(x_1), \ldots, p(x_n))$.
A transport plan between two positive measures $\vector{p}, \vector{q} \in \mathbb{R}^n_+$
is represented as a matrix $\bm{\pi} \in \mathbb{R}^{n \times n}_+$, where $\pi_{ij}$
denotes the density of mass transported from $x_i$ to $x_j$.
This gives the area-weighted discrete entropic OT problem $\mathrm{OT}_{C,\varepsilon}^\tau(\vector{p}, \vector{q})$ as
\begin{equation}
\label{eq:discrete-reg-ot}
\min_{\bm{\pi} \geq 0}
\sum_{i,j} C(x_i, x_j)\,\pi_{ij}\,a_i a_j
+ \varepsilon\, \textrm{KL}_{\vector{a}}(\bm{\pi} \mid \vector{p} \otimes \vector{q}) + \tau \left( \textrm{KL}_{\vector{a}}(\bm{\pi} \vector{a} \mid \vector{p})
  + \textrm{KL}_{\vector{a}}(\bm{\pi}^\top \vector{a}\mid \vector{q})\right),
\end{equation}
where
\[
 \textrm{KL}_{\vector{a}}(\bm{\pi} \mid \vector{p} \otimes \vector{q}) = -\sum_{i,j} \pi_{ij}\,\left(\log \frac{\pi_{ij}}{p_iq_j} - 1\right)\,a_i a_j
\]
and
\[
 \textrm{KL}_{\vector{a}}(\bm{\pi} \vector{a} \mid \vector{p}) = -\sum_i (\bm{\pi}\vector{a})_i\,\left(\log \frac{(\bm{\pi}\vector{a})_i}{p_i} - 1\right)\,a_i.
\]

We now show that~\eqref{eq:discrete-reg-ot} is equivalent to the standard discrete
entropic OT problem on the probability mass vectors
$\bar{\vector{p}} = \mathrm{diag}(\vector{a})\vector{p}$, and
$\bar{\vector{q}} = \mathrm{diag}(\vector{a})\vector{q}$. Let $\bm{\gamma} = \mathrm{diag}(\vector{a})\,\bm{\pi}\,\mathrm{diag}(\vector{a})$, then~\eqref{eq:discrete-reg-ot} is equivalent to
\begin{equation}
\label{eq:discrete-reg-ot-unbalanced-standard}
\min_{\bm{\gamma} \geq 0}
\sum_{ij} C(x_i, x_j)\,\gamma_{ij}
+ \varepsilon\,\mathrm{KL}_{\mathbf{1}}(\bm{\gamma} \mid \bar{\vector{p}} \otimes \bar{\vector{q}})
+ \tau\bigl(\mathrm{KL}_{\mathbf{1}}(\bm{\gamma}\mathbf{1} \mid \bar{\vector{p}})
           + \mathrm{KL}_{\mathbf{1}}(\bm{\gamma}^\top\mathbf{1} \mid \bar{\vector{q}})\bigr).
\end{equation}

Under the change of variables $\gamma_{ij} = a_i a_j \pi_{ij}$
(i.e.\ $\pi_{ij} = \gamma_{ij}/(a_i a_j)$), the three terms transform as follows.
The cost satisfies $\sum_{ij} C(x_i,x_j)\,\pi_{ij}\,a_i a_j = \sum_{ij} C(x_i,x_j)\,\gamma_{ij}$ directly.
The regularization becomes
\[
\textrm{KL}_{\vector{a}}(\bm{\pi} \mid \vector{p}\otimes\vector{q})
= \sum_{ij}a_i a_j\,\pi_{ij}\left(\log\frac{\pi_{ij}}{p_i q_j} - 1\right)
= \sum_{ij}\gamma_{ij}\left(\log\frac{\gamma_{ij}}{\bar{p}_i\bar{q}_j} - 1\right)
= \textrm{KL}_{\mathbf{1}}(\bm{\gamma} \mid \bar{\vector{p}}\otimes\bar{\vector{q}}).
\]
Finally, since $(\bm{\pi}\vector{a})_i = (\bm{\gamma}\mathbf{1})_i/a_i$,
\[
\textrm{KL}_{\vector{a}}(\bm{\pi}\vector{a} \mid \vector{p})
= \sum_i a_i\,(\bm{\pi}\vector{a})_i\left(\log\frac{(\bm{\pi}\vector{a})_i}{p_i} - 1\right)
= \sum_i (\bm{\gamma}\mathbf{1})_i\left(\log\frac{(\bm{\gamma}\mathbf{1})_i}{\bar{p}_i} - 1\right)
= \textrm{KL}_{\mathbf{1}}(\bm{\gamma}\mathbf{1} \mid \bar{\vector{p}}),
\]
and symmetrically $\textrm{KL}_{\vector{a}}(\bm{\pi}^\top\vector{a} \mid \vector{q})
= \textrm{KL}_{\mathbf{1}}(\bm{\gamma}^\top\mathbf{1} \mid \bar{\vector{q}})$.
Hence~\eqref{eq:discrete-reg-ot} and~\eqref{eq:discrete-reg-ot-unbalanced-standard} coincide.

We use Sinkhorn algorithm to solve the area-weighted discrete entropic OT problem~\eqref{eq:discrete-reg-ot}. Once it has converged, we recover the optimal Kantorovich potentials $\vector{\phi} = \varepsilon \log \vector{u}$ and $\vector{\psi} = \varepsilon \log \vector{v}$ and compute the transport cost as
\begin{equation}
  \mathrm{OT}_{C,\varepsilon}^{\tau}(\vector{p}, \vector{q})
  = \sum_i a_i\, \tau\!\left(1 - e^{-\phi_i/\tau}\right) p_i
  + \sum_i a_i\, \tau\!\left(1 - e^{-\psi_i/\tau}\right) q_i.
\end{equation}
This follows from strong duality for unbalanced entropic OT~\cite{chizat2018scaling}. In the balanced case $\tau=\infty$, the formula reduces to the standard dual $\sum_i a_i\,\phi_i\,p_i + \sum_i a_i\,\psi_i\,q_i$.

\subsection{Convergence Criterion}\label{app:convergence_criterion}

We monitor convergence via the relative left marginal violation. After each pair of Sinkhorn updates, we evaluate
\begin{equation}
  e_{\mathrm{marg}} = \frac{\bigl\|\vector{u} \odot \mathrm{HeatConv}(\vector{v};\,\{h_\ell\}) - \vector{p}\bigr\|_2}{\|\vector{p}\|_2},
\end{equation}
and declare convergence when $e_{\mathrm{marg}} < \delta$. To reduce overhead, this check is performed every $C$ iterations. We use $C = 10$ throughout our experiments.

\subsection{Numerical instabilities}\label{appendix:numerical_instabilites}
\label{appendix:instabilities}

In \texttt{PyTorch} with \texttt{float64} arithmetic, the smallest representable positive number is approximately $4.94\times10^{-324}$.
On $\mathbb{S}^2$, the squared-distance matrix $C_{ij}=d(x_i,x_j)^2$ has entries as large as $\pi^2$. 
The Sinkhorn algorithm requires evaluating the Gibbs kernel $\mathcal{K}_{ij}=e^{-C_{ij}/\varepsilon}$, whose entries underflow to zero as soon as $d(x_i,x_j)^2/\varepsilon \geq \ln\!\left(1/4.94\times10^{-324}\right) \approx 744.4$. 
Antipodal pairs are the first to be affected, at $\varepsilon \leq \pi^2/744.4 \approx 0.013$. As $\varepsilon$ decreases further, the effective support of the kernel around each point shrinks: progressively closer pairs underflow in turn until $\mathcal{K}$ is effectively diagonal. 
At that point Sinkhorn iterations can no longer redistribute mass and the algorithm degenerates.

A distinct instability arises from the truncated spherical harmonic representation. 
The Sinkhorn in Heat convolution step evaluates $\mathcal{F}^{-1}(\mathcal{F}[u](l,m) \cdot e^{-\frac{\varepsilon}{4}\ell(\ell+1)})$, which approximates the convolution by the heat kernel by a finite-degree expansion up to $L$. 
Although the exact heat kernel is a strictly positive operator, its degree-$L$ truncation can take small negative values when $\varepsilon^{-1/2} \gtrsim L$, i.e.\ when the kernel is too narrow to be resolved on the SH grid. 
Even a single negative entry in the reconstructed convolution suffices to make the Sinkhorn update undefined.

\section{Additional Numerical Experiments}
\subsection{Runtime comparison}\label{appendix:runtime}

We use the same setup as in the runtime comparison of the main paper (random probability measures on $\S^2$, single NVIDIA A40 GPU, \texttt{float64} precision, 5 runs per configuration), but vary the regularization parameter $\varepsilon$ across multiple values instead of fixing it to $0.1$.
Figure~\ref{fig:runtime_appendix} shows that the speed and scalability and advantage of \textsc{SHOT} observed in Figure~\ref{fig:runtime} is consistent across values of $\varepsilon$ and grid types.

\begin{figure}[t]
    \centering
    \includegraphics[width=\textwidth]{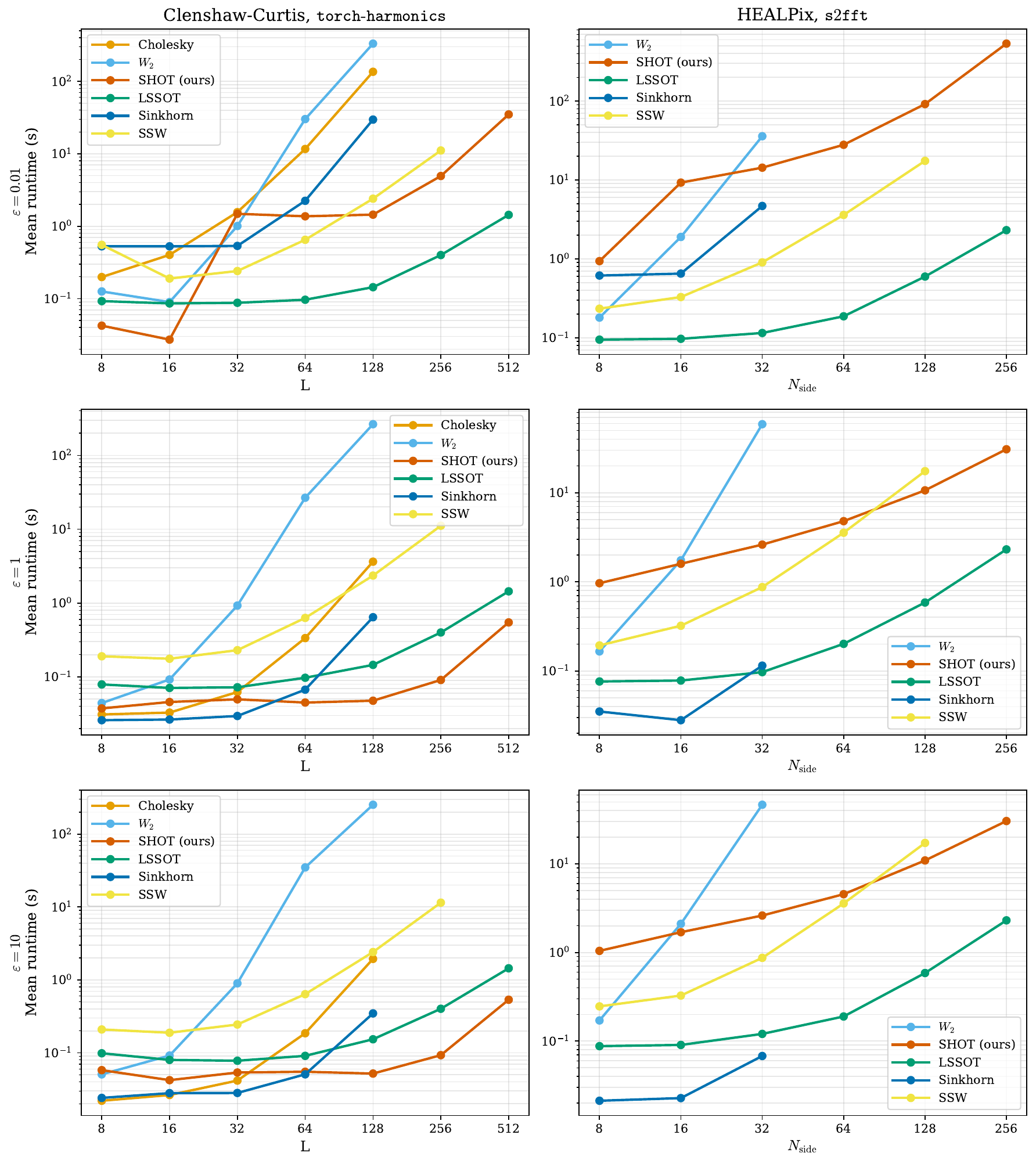}
    \caption{Runtime comparison across methods, grid types, and regularization parameters $\varepsilon$.}
    \label{fig:runtime_appendix}
\end{figure}

\subsection{Convolution on $\mathbb{S}^2$}
\label{appendix:convolution_sphere}

We validate that the heat-kernel convolution and the explicit matrix convolution converge to the same result as the grid resolution increases.
The input distribution is a Von Mises-Fisher (vMF) measure with concentration $\kappa = 10$ and a fixed random mean direction.
All experiments are run on CPU with \texttt{float64} precision.

\paragraph{Kernels and normalization.}
The two convolutions operate with kernels at different scales and must be brought to a common scale before comparison.
The \emph{matrix} convolution applies the unnormalized geodesic kernel $K_\varepsilon(x,y) = e^{-d(x,y)^2/\varepsilon}$, which integrates to
\[
  C(\varepsilon) = \int_{\S^2} e^{-d(x,y)^2/\varepsilon}\,\mathrm{d}\sigma(y),
\]
a quantity that depends on $\varepsilon$ and goes to $0$ as $\varepsilon \to 0$.
The \emph{heat} convolution applies the heat kernel $H_t$ at time $t = \varepsilon/4$, which is a \emph{stochastic} (probability-preserving) kernel satisfying $\int_{\S^2} H_t(x,y)\,\mathrm{d}\sigma(y) = 1$.
By Varadhan's formula, $H_{\varepsilon/4}(x,y) \approx e^{-d(x,y)^2/\varepsilon} / C(\varepsilon)$ for small $\varepsilon$, so the two outputs are related by
\[
  (H_{\varepsilon/4} * f)(x) \approx \frac{(K_\varepsilon * f)(x)}{C(\varepsilon)}.
\]
Comparing them directly would therefore show a systematic amplitude difference of $C(\varepsilon)$, which grows large as $\varepsilon \to 0$, rather than a discretization error.
To isolate the approximation error, we rescale the heat output by $C(\varepsilon)$ before computing the mean absolute error.
Note that this rescaling is irrelevant in the Sinkhorn iterations, where scaling $K$ by any constant is absorbed by the dual variables and leaves the optimal transport plan unchanged.

\paragraph{Setup and results.}
We consider two discretizations of $\S^2$: the Clenshaw-Curtis equiangular grid (via \texttt{torch-harmonics}) with band limits $L \in \{8, 16, 32, 64\}$, and the HEALPix grid (via \texttt{s2fft}) with $N_\mathrm{side} \in \{4, 8, 16, 32\}$.
For each $\varepsilon \in [10^{-3}, 10]$ (8 values log-spaced), we apply both convolutions and measure their mean absolute error after rescaling.
As expected, the error decreases as $L$ (resp.\ $N_\mathrm{side}$) grows, and is larger at large $\varepsilon$ where the heat-kernel approximation is less accurate.
Figure~\ref{fig:vmf_conv_quant} shows the quantitative error curves for both grids, and Figures~\ref{fig:vmf_conv_qual_th} and~\ref{fig:vmf_conv_qual_hp} show the qualitative Mollweide maps for the Clenshaw-Curtis and HEALPix grids respectively.

\begin{figure}[h]
  \centering
  \begin{subfigure}[t]{0.48\linewidth}
    \centering
    \includegraphics[width=\linewidth]{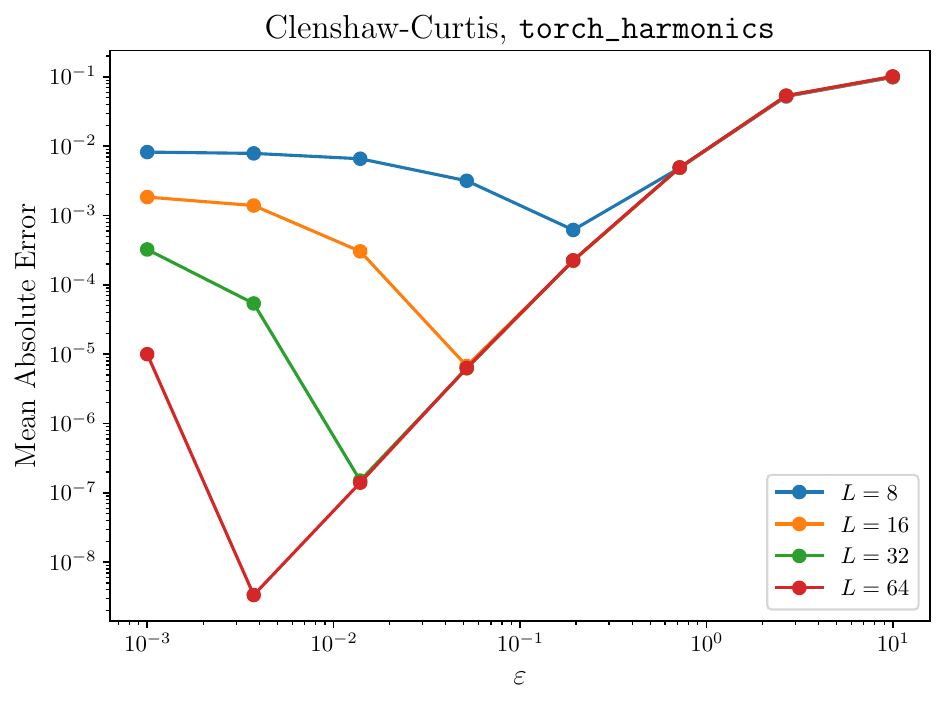}
    \caption{Clenshaw-Curtis ({\tt torch-harmonics}), $L \in \{8,16,32,64\}$.}
  \end{subfigure}
  \hfill
  \begin{subfigure}[t]{0.48\linewidth}
    \centering
    \includegraphics[width=\linewidth]{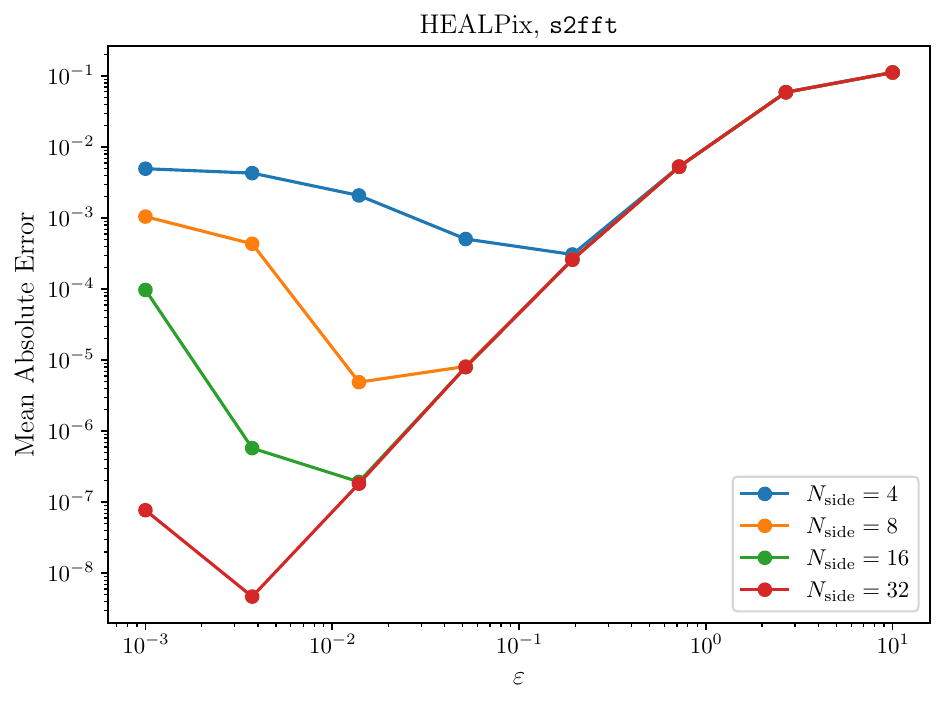}
    \caption{HEALPix ({\tt s2fft}), $N_\mathrm{side} \in \{4,8,16,32\}$.}
  \end{subfigure}
  \caption{Mean absolute error between heat and matrix convolutions as a function of $\varepsilon$, for both grid types.}
  \label{fig:vmf_conv_quant}
\end{figure}

\begin{figure}[p]
  \centering
  \includegraphics[width=\linewidth, height=0.85\textheight, keepaspectratio]{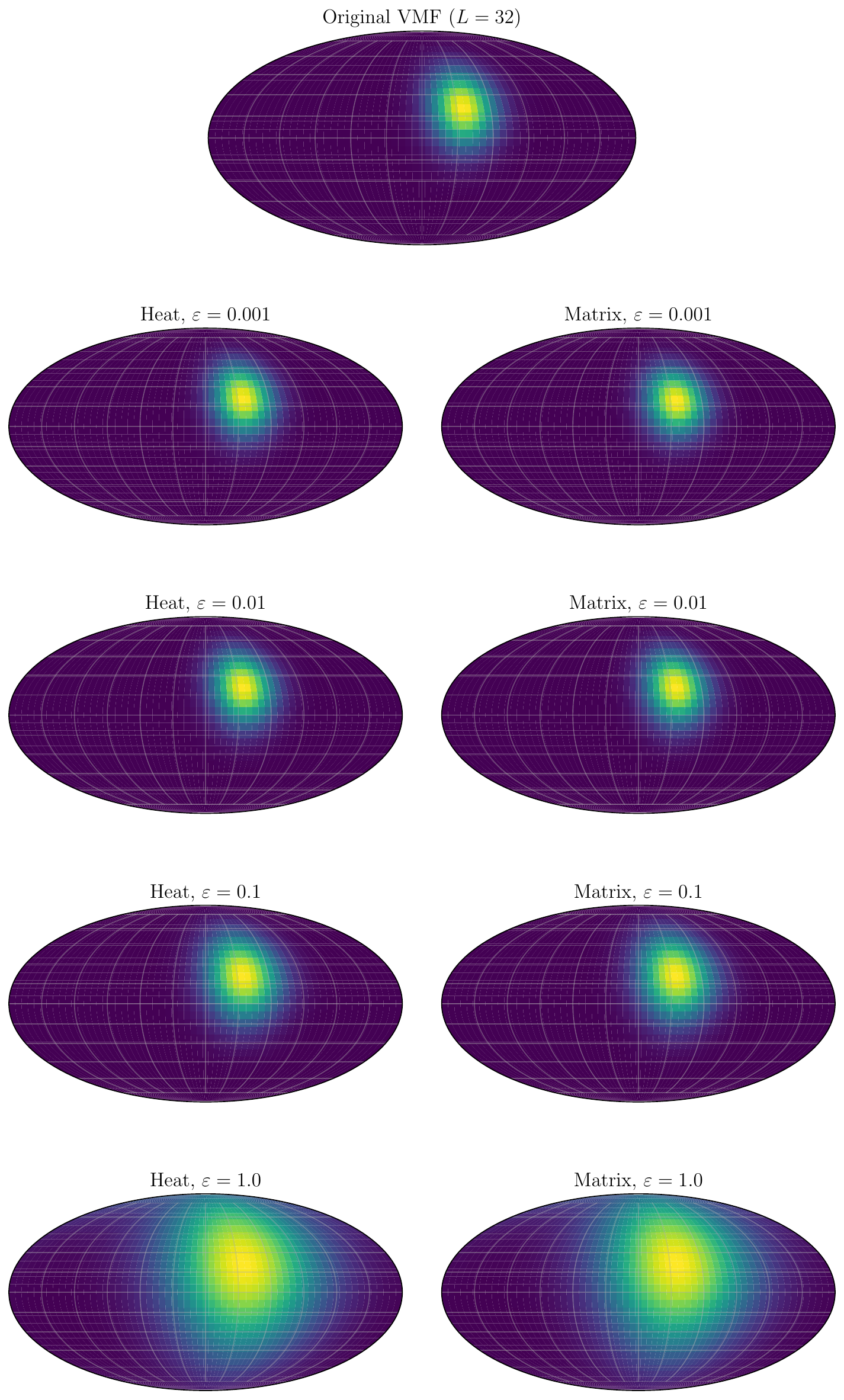}
  \caption{Mollweide projections of the heat (left column) and matrix (right column) convolutions of a vMF measure on the Clenshaw-Curtis grid ($L=32$), for $\varepsilon \in \{10^{-3}, 10^{-2}, 10^{-1}, 10^0, 10^1\}$ (top to bottom).}
  \label{fig:vmf_conv_qual_th}
\end{figure}

\begin{figure}[p]
  \centering
  \includegraphics[width=\linewidth, height=0.85\textheight, keepaspectratio]{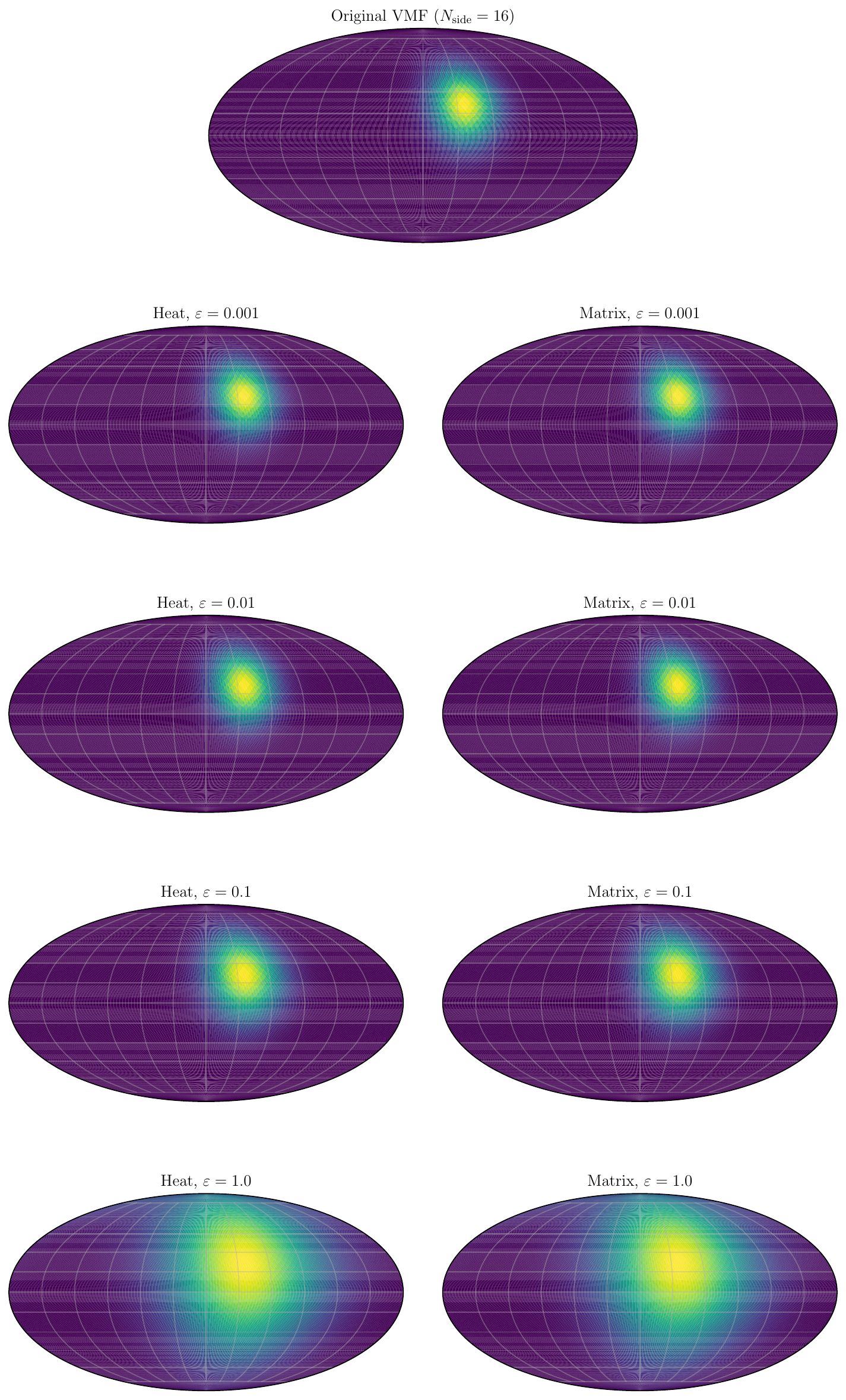}
  \caption{Mollweide projections of the heat (left column) and matrix (right column) convolutions of a vMF measure on the HEALPix grid ($N_\mathrm{side}=16$), for $\varepsilon \in \{10^{-3}, 10^{-2}, 10^{-1}, 10^0, 10^1\}$ (top to bottom).}
  \label{fig:vmf_conv_qual_hp}
\end{figure}

\section{Details on Climate Models Comparisons}\label{appendix:iclimate}

Precipitation fields are represented as
discrete measures on the \textsc{HEALPix} grid~\cite{gorski2005healpix}
at resolution $N_\mathrm{side} = 64$ ($\sim$49\,152 equal-area pixels,
$\approx 1^\circ$ effective resolution), which provides an
equal-area pixelation of $\mathbb{S}^2$ directly compatible with our
algorithm. ERA5 data is accessed via the HERA5
dataset\footnote{\url{https://orcestra-campaign.org/hera5.html}}, already provided in \textsc{HEALPix} format,
while CMIP6 model outputs are retrieved from the Pangeo Google Cloud
archive\footnote{\url{https://pangeo-data.github.io/pangeo-cmip6-cloud/}} and regridded to the same
\textsc{HEALPix} grid via bilinear interpolation. All fields are
converted to mm/day. We focus on total precipitation and analyze two
climatological seasons: December-January-February~(DJF) and
June-July-August~(JJA), each computed as a temporal average over
the 2000-2014 historical period. Visual illustrations of the resulting
precipitation distributions for ERA5 and one model (ACCESS-CM2) are shown in
Figure~\ref{fig:illustration}.
\begin{figure}[t]
   \centering\includegraphics[width=0.99\linewidth]{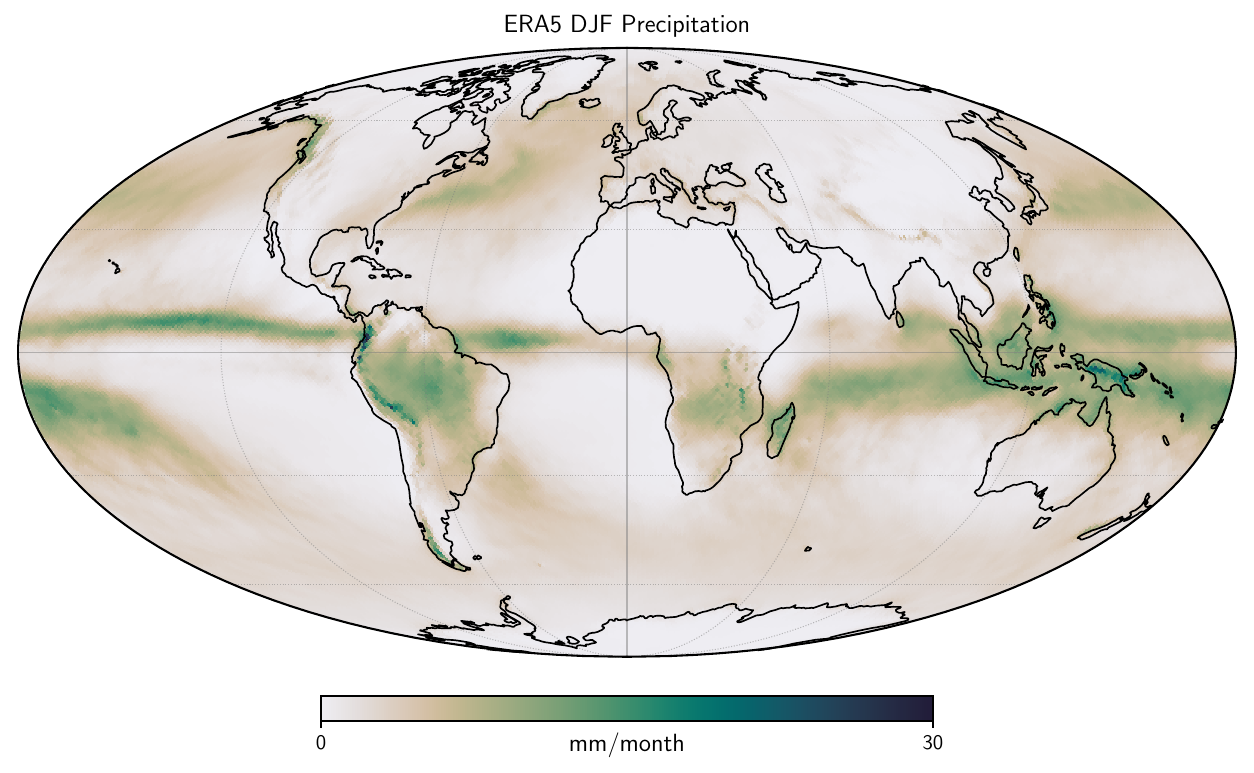}\\
   \centering\includegraphics[width=0.99\linewidth]{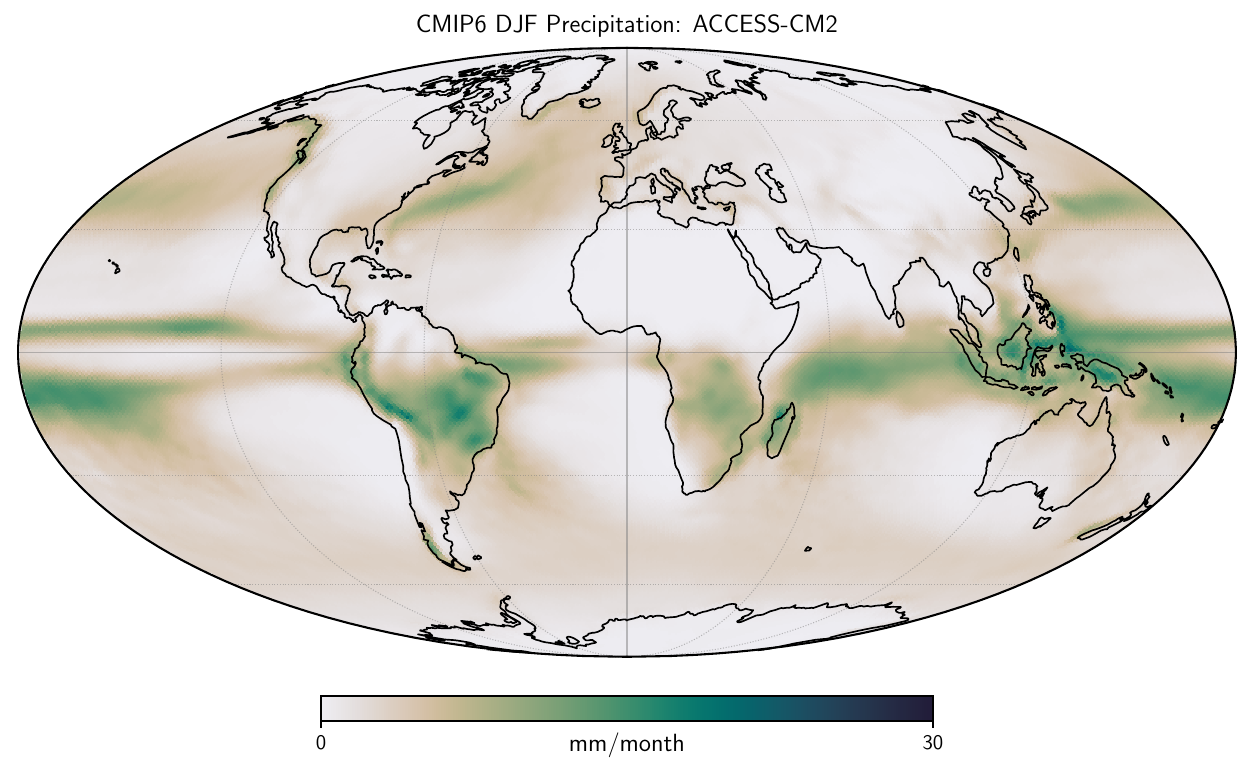}
  \caption{Illustrations of the precipitation distributions for the ERA5 reanalysis (top) and the ACCESS-CM2 model (bottom) during the DJF season.
  }
  \label{fig:illustration}
\end{figure}
The choice of precipitation as the variable of interest is motivated by its critical role in the climate system and its known challenges for climate models, which often struggle to accurately represent its spatial distribution and intensity.

\paragraph{Gradient-based regional attribution.}
We partition $\mathbb{S}^2$ into five geographic regions:
Tropics ($|\text{lat}| < 23.5^\circ$), ITCZ ($|\text{lat}| < 10^\circ$),
NH Midlatitudes ($30^\circ$--$60^\circ$N), SH Midlatitudes ($30^\circ$--$60^\circ$S),
and Arctic ($> 66.5^\circ$N), and compute the area-weighted RMS of
$g_m$ over each region as a regional bias index. These indices are
reported for all eight models in both seasons in
Figure~\ref{fig:regional_gradients}, normalized by the global maximum across
all models, regions, and seasons to preserve inter-model
differences.

\paragraph{Results.}
Four physically interpretable patterns emerge from
Figure~\ref{fig:regional_gradients}.

\textit{Tropical and ITCZ biases are structurally season-independent.}
These two regions show the largest absolute gradient norms in both
seasons and the smallest DJF$-$JJA differences ($|\Delta| \leq
0.02$ for all models), confirming that tropical convection
deficiencies in CMIP6 are a structural property of model
parameterizations rather than a seasonally forced error. This
has a direct practical implication: improving tropical
precipitation requires architectural changes to convection schemes,
not seasonal retuning.

\textit{SH midlatitudes are uniformly well captured.}
The SH midlatitude column shows the lowest gradient norms across
both seasons ($|\Delta| \leq 0.01$) and negligible inter-model
spread, indicating that Southern Ocean precipitation is similarly
represented by all models and that the unbalanced Sinkhorn
divergence adds no discriminating power in this region. This is
itself a useful negative result.

\textit{NH midlatitudes show a mild but consistent DJF elevation.}
Most models exhibit slightly larger NH midlatitude bias in DJF than
JJA ($\Delta = +0.02$ for ACCESS-CM2, GFDL-ESM4, IPSL-CM6A-LR,
MRI-ESM2-0; $\Delta = 0.00$ for MPI-ESM1-2-HR and CanESM5),
consistent with the activation of boreal winter storm tracks and
polar vortex dynamics. The signal is secondary and not universal,
suggesting model-dependent sensitivity to wintertime circulation.

\begin{figure}[!t]
  \centering
  \includegraphics[width=0.99\linewidth]{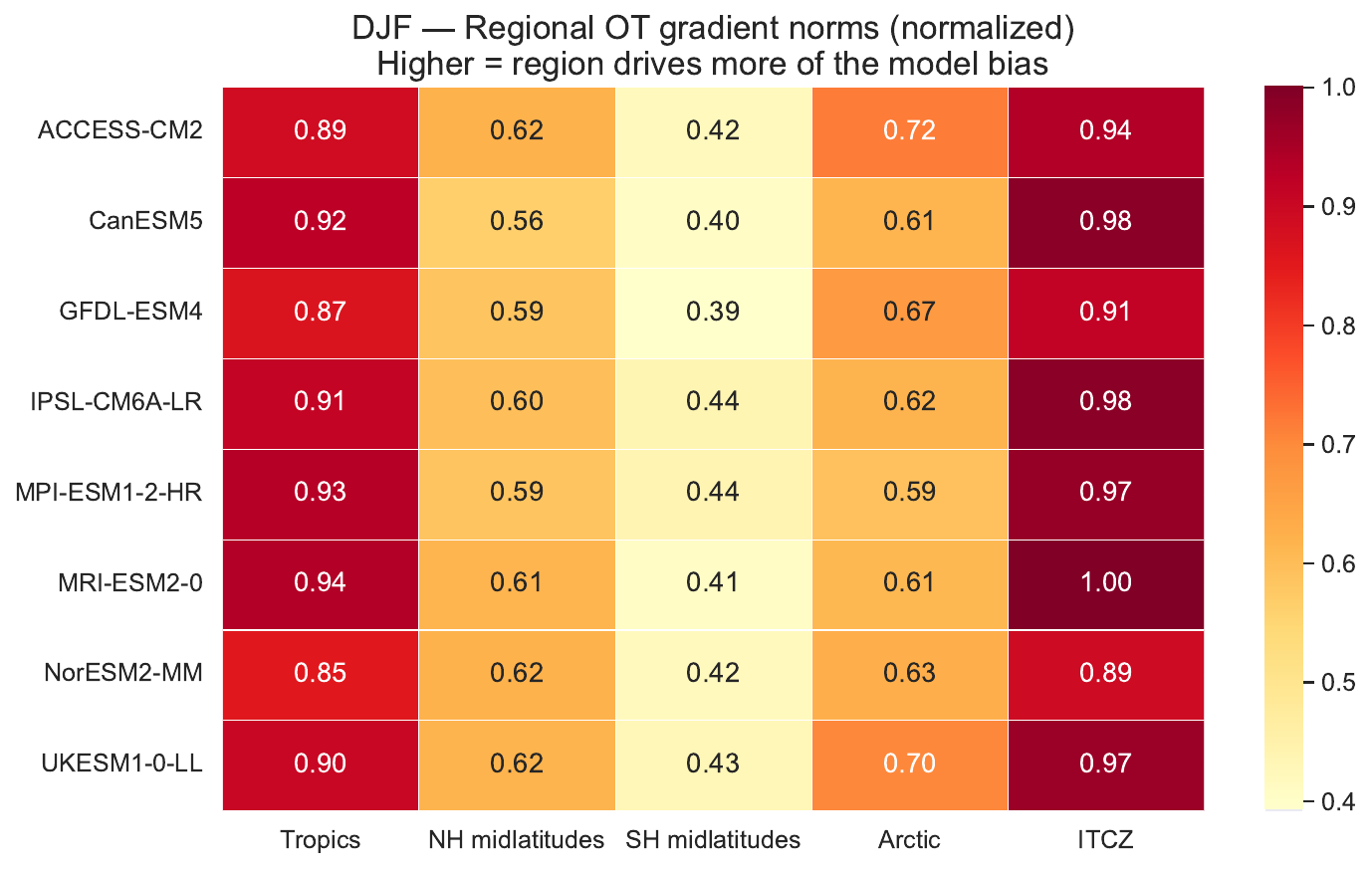}\\
   \centering
  \includegraphics[width=0.99\linewidth]{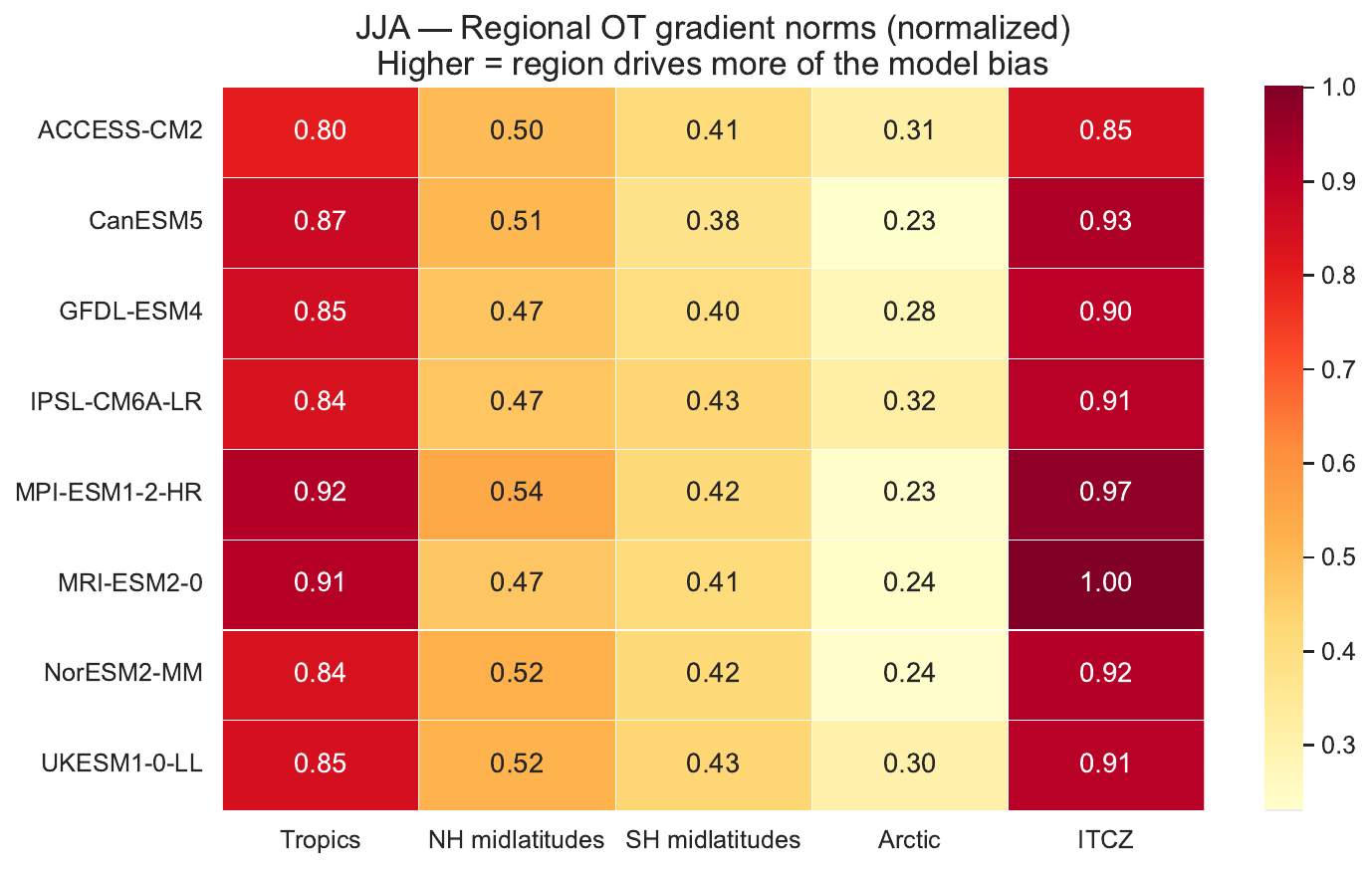}
  \caption{
    Regional OT gradient norms for eight CMIP6 models in DJF
    (up) and JJA (down), normalized by the global maximum across
    all models, regions, and seasons. Higher values indicate that
    a region contributes more to the total Sinkhorn divergence
    from ERA5. Tropical and ITCZ biases dominate in both seasons
    with consistent inter-model spread. The Arctic column is
    strongly elevated in DJF relative to JJA, while SH midlatitudes
    remain uniformly low across all models and seasons.
  }
  \label{fig:regional_gradients}
\end{figure}


\end{document}